\documentclass[a4paper]{article}

\usepackage[margin=1.4in]{geometry}
\usepackage{setspace}
\usepackage{amsmath, amssymb, amsthm, amsfonts}
\usepackage{pifont} 
\usepackage{libertine}
\usepackage{natbib}
\usepackage{authblk} 
\usepackage{adjustbox}
\usepackage{rotating}
\usepackage{parskip}
\usepackage[T1]{fontenc}   
\usepackage[nopatch=footnote]{microtype}        
\usepackage{url}           
\usepackage{booktabs}        
\usepackage[dvipsnames]{xcolor}    
\usepackage{graphicx}
\usepackage{wrapfig}
\usepackage{subfig}
\usepackage{tikz}
\usepackage{pgfplots}
\pgfplotsset{compat=1.18}
\usepgfplotslibrary{fillbetween}
\usepackage{comment}
\usepackage{times}
\usepackage{bbm}
\usepackage{mathabx}
\usepackage{mathtools}
\usepackage{makecell}
\usepackage{enumitem}
\newlist{defenum}{enumerate}{1}
\setlist[defenum]{label=\upshape\theassumption.\arabic*}
\usepackage[hidelinks]{hyperref}  
\usepackage[capitalize,noabbrev]{cleveref}
\hypersetup{colorlinks={true},linkcolor={BrickRed},citecolor={CadetBlue}}

\newcommand{\myitem}[1]{
\item[#1]\protected@edef\@currentlabel{#1}
}

\newcommand{\mathset}[1]{\left\{#1\right\}}
\newcommand{\pr}[1]{\left(#1\right)}
\newcommand{\br}[1]{\left[#1\right]}
\newcommand{\abs}[1]{\left\lvert#1\right\rvert}
\newcommand{\norm}[1]{\left\|#1\right\|}
\newcommand{\setmid}[0]{\;\middle|\;}

\newcommand{\floor}[1]{\left\lfloor #1 \right\rfloor}
\newcommand{\inner}[2]{\left\langle #1,#2\right\rangle}

\newcommand{\A}[0]{\mathbf{A}}
\newcommand{\I}[0]{\mathbf{I}}
\newcommand{\X}[0]{\mathbf{X}}
\newcommand{\e}[0]{\mathbf{e}}
\newcommand{\zeros}[0]{\mathbf{0}}

\newcommand{\EE}[0]{\mathbb{E}}
\newcommand{\NN}[0]{\mathbb{N}}
\newcommand{\RR}[0]{\mathbb{R}}
\newcommand{\PP}[0]{\mathbb{P}}

\newcommand{\cC}[0]{\mathcal{C}}

\newcommand{\cE}[0]{\mathcal{E}}

\newcommand{\cI}[0]{\mathcal{I}}

\newcommand{\cM}[0]{\mathcal{M}}
\newcommand{\cN}[0]{\mathcal{N}}

\newcommand{\cR}[0]{\mathcal{R}}
\newcommand{\cV}[0]{\mathcal{V}}

\newcommand{\eps}[0]{\varepsilon}

\newcommand{\ha}[0]{\widehat{\alpha}}

\newcommand{\tstar}{t^\star}
\newcommand{\alphastar}{\alpha^\star}

\newcommand{\alphat}[0]{\alpha_t}
\newcommand{\alphatstar}[0]{\alpha_{t^\star}}

\DeclareMathOperator{\LSEOp}{LSE}
\newcommand{\haLSE}[0]{\ha_{{\LSEOp}}}

\newcommand{\Kt}[0]{K_\tau}

\DeclareMathOperator{\diamOp}{diam}
\newcommand{\diam}[0]{{\diamOp}}
\DeclareMathOperator{\locOp}{loc}

\newcommand{\rstar}[0]{r_\star}

\newcommand{\rzero}[0]{r_{0}}
\newcommand{\locEntropy}[0]{M^{\locOp}}

\DeclareMathOperator{\signOp}{sign}
\newcommand{\sign}[1]{\signOp\pr{#1}}

\DeclareMathOperator{\rankOp}{rk}
\newcommand{\rank}[1]{\rankOp\pr{#1}}

\DeclareMathOperator*{\argmax}{arg\,max}
\DeclareMathOperator*{\argmin}{arg\,min}

\newcommand{\ind}[0]{\mathbf{1}}

\DeclareMathOperator{\arcsinh}{arcsinh}
\DeclareMathOperator{\diagOp}{diag}
\newcommand{\diag}[0]{{\diagOp}}

\newcommand{\Rhat}[0]{\widehat{\cR}}
\newcommand{\Rbar}[0]{\overline{\cR}}

\newtheoremstyle{thmstyle}
  {6pt} 
  {2pt} 
  {\itshape} 
  {} 
  {\bfseries} 
  {.} 
  {.5em} 
  {} 

\newtheoremstyle{defstyle}
  {6pt} 
  {2pt} 
  {} 
  {} 
  {\bfseries} 
  {.} 
  {.5em} 
  {} 

\theoremstyle{thmstyle}
\newtheorem{theorem}{Theorem}
\newtheorem*{theorem*}{Theorem}
\newtheorem{proposition}[theorem]{Proposition}
\newtheorem{lemma}{Lemma}
\newtheorem{corollary}{Corollary}

\theoremstyle{defstyle}
\newtheorem{assumption}{Assumption}
\renewcommand*{\theassumption}{\Alph{assumption}}
\newtheorem{definition}{Definition}

\theoremstyle{remark}
\newtheorem{remark}{Remark}

\title{Sharp Risk Bounds for Early-Stopping in Gaussian \\ Linear Regression}
\author[1]{Tobias Wegel\thanks{Corresponding author: \texttt{twegel@ethz.ch}.}}
\author[1]{Gil Kur}
\author[2]{Patrick Rebeschini}
\affil[1]{Department of Computer Science, ETH Zürich, Switzerland}
\affil[2]{Department of Statistics, University of Oxford, UK}
\date{\today}

\begin{document}

\maketitle

\begin{abstract}
    \noindent 
    We study early-stopped mirror descent (ESMD) for high-dimensional Gaussian linear regression over arbitrary convex bodies and design matrices, where the task is to minimize the in-sample mean squared error. Our main result shows that some of the sharpest risk bounds for the least squares estimator (LSE), based on the local Gaussian width, extend to ESMD.
We derive sufficient conditions on the potential, expressed via the Minkowski functional, under which our result holds. These conditions allow us to construct new potentials and analyze existing ones. Our results then yield general sufficient conditions for minimax optimality of ESMD, provide a systematic comparison with the LSE, and establish the tightest known risk bound in the $\ell_1$-constrained setting.
\end{abstract}

\section{Introduction}

Regularization methods generally fall into two categories: \emph{Explicit} regularization, where the learning objective is altered to reduce model complexity via constraints or penalty terms, and \emph{implicit} regularization, where the optimization solver inherently controls model complexity via algorithmic primitives and parameter tuning. A commonly used implicit regularization technique is to stop iterative optimization algorithms before convergence. This is called \emph{early-stopping} or \emph{iterative regularization}, and has been investigated for different settings and algorithms such as linear and kernel regression with coordinate descent \citep{Hastie2001statisticallearning, Efron2004Least, Rosset2004boosting, Zhang2005boosting}, gradient descent \citep{Ali2019continuous,Buhlmann2003boosting,Yao2007,Raskutti2011early,Bauer2007regularization}, primal-dual gradient methods \citep{Molinari2020iterativeregularizationconvexregularizers,Molinari2024iterative} and algorithms based on factorized parameterizations \citep{Gunasekar2017Implicit,Li2018Algorithmic,Vaskevicius2019ImplicitRegularization,Zhao2022high}.
One benefit of iterative regularization is the simultaneous study of modeling and numerical aspects; often iterative regularization improves computational efficiency while retaining good statistical performance \citep{Yao2007}.

A recurring approach to understanding iterative regularization methods is to tie them to ``corresponding'' explicit regularization methods or constraint geometries.
Usually, the stopping time then takes the role of regularization strength, analogous to the (inverse) coefficient of the penalty in explicit regularization. 
For example, coordinate descent has been shown to be related to explicit $\ell_1$-regularization  \citep{Hastie2001statisticallearning, Efron2004Least, Rosset2004boosting, Zhang2005boosting} and gradient descent has been shown to trace the path of ridge regularization \citep{Bauer2007regularization,Ali2019continuous}. 
However, most results on early-stopping fall into at least one of three categories: either the risk bounds are \emph{un}localized (e.g., for online mirror descent \citep{Shalev2007online,bach2024learning}), they only hold in the low-dimensional regime (e.g., \citet{Suggala2018connecting}), or they use tools (such as spectral analysis) that only apply to specific geometries like $\ell_2$ or Hilbert spaces (e.g., \citet{Wei2018early,Ali2019continuous,Yao2007}).
In particular, a sharp localized analysis of early-stopping for general geometries in high dimensions seems to be lacking.
Since the role of the geometry is particularly important to circumvent the curse of dimensionality, sharp localized risk bounds in this setting are of particular interest.

In this work, we address this gap in the high-dimensional linear regression setting using the framework introduced by \cite{Kanade2023earlystopped}. In this setting, we observe a design matrix $\X\in\RR^{n\times d}$ (with possibly $d \gg n$) and $n$ random responses from the linear model
\begin{equation}
    y=\X\alphastar+\xi\in \RR^n,
\end{equation}
with a linear ground-truth function parametrized by $\alphastar\in \RR^d$ and additive i.i.d.\ Gaussian noise $\xi\sim \cN(0,\I_{n})$.
We assume that the ground truth satisfies the shape constraint given by
\begin{equation}
\label{eq:convex-body}
    \quad \alphastar \in K_\tau := \tau K=\mathset{\tau\alpha\setmid \alpha\in K}\subset\RR^d,
\end{equation}
where $\tau>0$ is a ``radius'' and $K\subset\RR^d$ is any convex body, that is, it is convex, compact and the origin is contained in its interior.
We define the empirical and in-sample prediction risk using squared loss, respectively, as
\begin{align*}
     \Rhat(\ha):=\frac{1}{n}\norm{\X\ha-y}_2^2, 
     \quad \cR(\ha):= \frac{1}{n}\norm{\X(\ha-\alphastar)}_2^2.
\end{align*}
Under these assumptions, the aim of a predictor $\ha\equiv \ha(\X,y)\in\RR^d$ is to achieve minimal in-sample risk using the observations $\X$ and $y$, as well as knowledge of the convex body $K$ and the radius $\tau$. The knowledge of $\tau$ is not always required, which we specify in those instances. Notice that this setting can be viewed as a Gaussian sequence model over the convex constraints $\X \Kt$, see \citet{johnstone2017gaussian} for a detailed account.
Our results are restricted to in-sample prediction (also known as fixed design), see e.g., \citet[Sec. 14.1]{Wainwright2019} or \citet[Chp. 3]{bach2024learning} for a discussion of the differences to random design.

\paragraph{Notation.} The Bregman divergence of a strictly convex and differentiable function $\psi:\RR^d\to\RR$ is defined as 
$$D_\psi(\alpha,\alpha')=\psi(\alpha)-\psi(\alpha')-\inner{\nabla\psi(\alpha')}{\alpha-\alpha'}$$
for $\alpha,\alpha'\in\RR^d$.
We denote the Bregman ball as $B_\psi(\alpha',r)=\mathset{\alpha\in\RR^d\setmid D_\psi(\alpha',\alpha)\leq r}$
and $\ell_p$-norm balls as $B_p^d=\{\alpha\in\RR^d | \norm{\alpha}_p\leq 1\}$.
We write $a\lesssim b$ if there is a constant $C>0$ such that $a\leq Cb$, $a\asymp b$ if $a\lesssim b \lesssim a$, and $a\wedge b=\min\mathset{a,b}$, $a\vee b=\max\mathset{a,b}$. We use Minkowsi sum notation throughout: for a matrix $\X\in\RR^{n\times d}$, sets $K,K'\subset \RR^d$ and $v\in\RR^d$ we write $\X K:= \{\X\alpha\mid \alpha\in K\}$, $K+v := \{\alpha+v \mid \alpha\in K\}$ and $K+K':=\mathset{\alpha+\alpha':\alpha\in K, \alpha'\in K'}$.

\subsection{Local Gaussian Width and the LSE}

A natural and well-studied estimator in this setting is the constrained \emph{Least Squares Estimator} (LSE), also known as the maximum likelihood estimator. It is defined by minimizing the empirical risk over the constraint set from \eqref{eq:convex-body}, 
\begin{equation}
\label{eq:constrained-least-squares}
    \haLSE\in\argmin_{\alpha\in \Kt} \Rhat(\alpha).
\end{equation}
The predictions of the LSE on the sample $\X$ are given by the orthogonal projection of $y$ onto $\X \Kt=\mathset{\X\alpha\setmid \alpha\in \Kt}$, which is unique by the convexity of $\X \Kt$. Hence, while the minimizer in \cref{eq:constrained-least-squares} is not necessarily unique (especially in the high-dimensional setting where $d> n$), its predictions on the sample are, and consequently, we do not need to distinguish between the minimizers any further.

The \emph{Gaussian width} (cf.\ \citet{Vershynin2018}) of a set $K\subset \RR^d$ is defined with a Gaussian vector $\xi\sim \cN(0,\I_{d})$ as $$w(K):=\EE_\xi\br{\sup_{\theta\in K}\inner{\xi}{\alpha}}.$$
In his seminal paper, \cite{Chatterjee2014} showed that the risk of the LSE concentrates sharply around a \emph{critical radius} that maximizes a function of the \emph{local Gaussian width}:
\begin{definition}
\label{def:critical-radius}
For $\alpha\in K\subset\RR^d$, define the function $f_{\alpha,K}:[0,\infty)\to [-\infty,\infty)$ as $$f_{\alpha,K} (r) := w\pr{ (K-\alpha) \cap rB^d_2}- \frac{r^2}{2}.$$
The \emph{critical} and the \emph{stationary radius} of a set $K\subset\RR^d$ around a point $\alpha\in K$ are defined, respectively, as
\begin{equation}
\label{eq:radii-definition}
    \begin{aligned}
        \rstar(\alpha,K):=\argmax_{r\geq 0}f_{\alpha,K}(r) \quad \text{and} \quad \rzero(\alpha,K):= \inf\mathset{r> 0 \setmid f_{\alpha,K}(r)\leq 0}.
    \end{aligned}
\end{equation}
We denote the maximal critical and stationary radii on $K$ as
\begin{align*}
    \rstar(K):= \sup_{\alpha\in K} \rstar(\alpha,K) \quad \text{and} \quad  \rzero(K):= \sup_{\alpha\in K} \rzero(\alpha,K).
\end{align*}
\end{definition}
\noindent Specifically, in our notation, \cite{Chatterjee2014} showed that $\cR(\haLSE)$ concentrates sharply around $\rstar^2(\X\alphastar,\X\Kt)/n$, up to constant factors. Multiple other works such as \cite{Bellec2016sharporacleinequalitiessquares,Prasadan2024some} bound $\cR(\haLSE)$ in terms of the stationary radius $\rzero^2(\X\alphastar,\X\Kt)$ instead, yielding tight leading constants.
Importantly, the stationary radius $\rzero(\alpha,K)$ can always be bounded by solving
\begin{equation}
\label{eq:localized-inequality}
    w\pr{(K-\alpha) \cap r B^d_2}\leq \frac{r^2}{2}
\end{equation}
for $r\geq 0$, where one trivial solution of \eqref{eq:localized-inequality} is always given by $r^2=2w(K)$.
Moreover, the stationary radius is an upper bound on the critical radius
\citep[Proposition 1.3]{Chatterjee2014}.
To summarize, if $r\geq 0$ is a solution to \eqref{eq:localized-inequality} for a convex body $K$ and $\alpha\in K$, it holds that 
\begin{equation}
\label{eq:radii-comparison}
    \rstar^2(\alpha,K)\leq \rzero^2(\alpha,K)\leq \min\mathset{r^2,2w(K)}.
\end{equation}

\subsection{Overview of Contributions}
\label{subsec:contributions}

We study \emph{mirror descent} \citep{Nemirovski1983}, which is a gradient-based iterative optimization method that generalizes gradient descent to different geometries.
\begin{definition}
\label{def:continuous-mirror-descent}
    Let $\psi:\RR^d\to \RR$ be twice differentiable and have positive definite Hessian $\nabla^2\psi$ everywhere. 
    Unconstrained \emph{continuous-time} \emph{mirror descent} using $\psi$ and initialized at $\alpha_0\equiv 0\in\RR^d$ is defined through the ODE
    \begin{equation}
        \label{eq:continuous-mirror-descent-definition}
        \frac{d}{dt}\alphat = -\pr{\nabla^2\psi(\alphat)}^{-1}\nabla \Rhat(\alphat).
    \end{equation}
\end{definition}
\begin{definition}
\label{def:discrete-mirror-descent}
Let $\psi:\RR^d\to \RR$ be  differentiable, strictly convex and let the gradient of $\psi$ be surjective, i.e., $\mathset{\nabla\psi(\alpha)\setmid \alpha\in\RR^d}=\RR^d$. 
Unconstrained \emph{discrete-time mirror descent} using $\psi$, initialization $\alpha_0\equiv 0\in\RR^d$ and fixed step-size $\eta>0$, is defined through the recursion
\begin{equation}
\label{eq:discrete-mirror-descent-definition}
    \nabla\psi (\alpha_{t+1})=\nabla\psi(\alphat)-\eta \nabla \Rhat(\alphat).
\end{equation}
\end{definition}
\noindent In both cases, the function $\psi$ is called the \emph{mirror map} or \emph{potential} of the mirror descent algorithm. If there are multiple minimizers of the empirical risk, the potential determines which of them mirror descent converges to \citep{Gunasekar2018}. More generally, it determines the optimization path, see \cref{subsec:risk-along-optimization-path} on page \pageref{subsec:risk-along-optimization-path} for a visualization.
We remark that continuous-time mirror descent is also referred to as Riemannian gradient flow \citep{Gunasekar2021mirrorless}. Stopping mirror descent before convergence, that is, using $\alphatstar$ for some $\tstar>0$ as the estimator is called \emph{Early-Stopped Mirror Descent} (ESMD).

So far, an analysis akin to the works of \citet{Chatterjee2014} and \citet{Bellec2016sharporacleinequalitiessquares} has eluded early-stopped mirror descent. In this work, we close this gap and show that a risk bound almost identical to the one from \cite{Bellec2016sharporacleinequalitiessquares} also applies to ESMD, provided the potential is appropriately chosen based on $K$ and $\tau$. This bound holds for \emph{any} convex constraints and \emph{any} design matrix, and importantly, in the high-dimensional setting ($d\gg n$).

We summarize our main contributions below.
\begin{itemize}[leftmargin=*,itemsep=2pt,topsep=2pt]
    \item We prove a tight bound on the in-sample risk of ESMD in terms of the stationary radius (\cref{thm:in-sample-risk-bound}). As a consequence, we provide sufficient conditions under which the worst-case risk of ESMD is bounded by that of the LSE (\cref{cor:comparison-LSE}) and for minimax optimality (\cref{cor:minimax-optimality}). 
    \item Using the Minkowski functional of the convex body, we provide sufficient conditions on the optimization potential for our bound to apply (\cref{ass:sufficient-conditions-psi}). We use these conditions for developing new (and analyzing existing) potentials in several examples (\cref{sec:applications}). 
    \item We apply our risk bounds to $\ell_p$-norm balls with $p\in [1,2)$ as well as general $M$-convex hulls and derive sharp statistical rates (\cref{sec:applications}). We accompany our bound for $p\in(1,2)$ with a matching minimax lower bound for column-normalized fixed design matrices. For $\ell_1$-constraints, our bound improves upon the best known bounds, demonstrating the benefits of our tight analysis.
\end{itemize}

\section{Related Work}
\label{sec:related-work}

\paragraph{Constrained Least Squares.} The statistical performance of the LSE under convex constraints is well-studied, for example, in \cite{Birge1993rates,Vershynin2015estimation,Bellec2016sharporacleinequalitiessquares,Plan2017high,kur2023variance}. 
Tight estimates of the local Gaussian width of specific convex bodies are established, for instance, in
\citet{Gordon2007,Bellec2017localizedgaussianwidthmconvex}.
The minimax rates of the Gaussian sequence model under convex constraints are characterized exactly in \cite{Neykov2022minimax}, and the minimax sub-optimality of the LSE for certain constraints has been described in \cite{Prasadan2024some}. 
When the convex body is an $\ell_1$-norm ball, LSE recovers the LASSO \citep{Tibshirani1996lasso} in constrained form, which has been studied extensively in \citet{Candes2005,Bunea2007,Ye2010rate,Buhlmann2011, Pathak2024design}.

\paragraph{Mirror Descent.} Mirror descent as an optimization procedure is well-studied \citep{Beck2003mirror,Shalev2007online,Agarwal2012information}. The implicit bias of mirror descent in the overparameterized regime is studied in \cite{Gunasekar2018, Sun2023unified}.
The generalization properties of \emph{online} mirror descent have been studied extensively \citep{Shalev2007online,Orabona2023modern, Lattimore2024bandit,bach2024learning}.
Notable instances of this are \cite{srebro2011universality,Levy2019necessary,gatmiry2024computing}, who also relate the potential to the geometry of the constraint set, even showing a ``universality'' of online mirror descent. While these bounds can be minimax optimal, they are usually not \emph{local} (with a few exceptions \citep{rakhlin2013localization}), or only give guarantees on aggregated (e.g., averaged) iterates, rather than an early-stopped iterate.

\paragraph{Local risk bounds for early-stopping.}
The application of localized complexity measures to iterative regularization methods is scarce in the literature \citep{Yao2007,Raskutti2011early}. While local Gaussian width appears in \citet{Wei2018early}, their results only apply to (unconstrained) RKHS. 
One could also obtain local risk bounds by directly tying mirror descent to explicit regularization paths, but known results either require strong convexity of the empirical risk, which is necessarily violated in high dimensions \citep[Section 4]{Suggala2018connecting}, or apply only to the $\ell_2$ geometry \citep{Ali2019continuous}.
In \cite{Kanade2023earlystopped}, a general analysis using offset Rademacher complexities is introduced; our work explicitly builds on their framework.

\section{Main Results}
\label{sec:main-result}

We now provide a list of sufficient conditions for the mirror descent potential based on the Minkowski functional of the convex body \citep{Bonnesen_1934}.

\begin{definition}
\label{def:Minkowski-functional}
For a convex body $K\subset\RR^d$, the function $\varphi_K:\RR^d\to\RR$, defined as
\begin{equation*}
    \varphi_K(\alpha):=\inf\mathset{\tau>0\setmid \alpha\in \tau K}
\end{equation*}
is called the \emph{Minkowski functional} of $K$ (also referred to as distance or gauge function).
\end{definition}
The Minkowski functional is a norm if and only if $K$ is centrally symmetric (i.e., $K = -K$). In that case, we use the notation $\norm{\cdot}_{K}$ rather than $  \varphi_K(\cdot)$.
Any convex body $K\subset \RR^d$ that contains the origin in its interior 
can be written as $K=\mathset{\alpha\in\RR^d\setmid \varphi_K(\alpha)\leq 1}$.
Noticably, convexity of $\Kt$ and strong Lagrange duality imply that we can rewrite the LSE in unconstrained form as 
\[
\haLSE \in \argmin_{\alpha\in\RR^d} \mathset{\Rhat(\alpha) + \lambda_n(\tau)\varphi_K(\alpha) }
\]
for some data-dependent, not necessarily computable regularization strength $\lambda_n(\tau) \geq 0$.
The conditions on the potential we formulate below then depend on whether discrete-time or continuous-time mirror descent is used.
\begin{assumption}
\label{ass:sufficient-conditions-psi}
    The potential $\psi:\RR^d \to [0,\infty)$ satisfies:
    \begin{defenum}[leftmargin=*,itemsep=0pt,topsep=2pt]
        \myitem{(I)} \label{ass:suff-1} For continuous-time MD, $\psi$ is twice differentiable, and in discrete-time, $\psi$ is differentiable.
        In both cases, the gradient of $\psi$ vanishes at zero, that is, $\nabla \psi(0)=0$.
        \myitem{(II)} \label{ass:suff-2} The square-root $\sqrt{\psi}$ is convex.
        \myitem{(III)} \label{ass:suff-3} In discrete time, $\psi$ is $\rho$-strongly convex with respect to some norm, and in continuous time, it is strictly convex (and in paricular, satisfies \cref{def:discrete-mirror-descent,def:continuous-mirror-descent}, respectively).
        \myitem{(IV)} \label{ass:suff-4}There exist constants $c_l,c_u>0$ independent of all other parameters, such that
        \begin{align*}
        \forall \alpha\in\RR^d&:\quad  \varphi_K(\alpha)\leq c_l \sqrt{\psi(\alpha)}, \\ 
        \forall \alpha\in \Kt&:\quad  \sqrt{\psi(\alpha)}\leq c_u \tau.
    \end{align*}
    \end{defenum}
\end{assumption}

\noindent Throughout this paper, we denote the ``approximation constant'' $c_a = c_l\cdot c_u$.
We remark that \cref{ass:sufficient-conditions-psi} is loosely connected to the notion of $\kappa$-regularity from \cite[Def. 2.1]{Juditsky2008large}.

If the squared Minkowski functional $\varphi_K^2$ satisfies \ref{ass:suff-1} and \ref{ass:suff-3}, and since it satisfies \ref{ass:suff-2} and  \ref{ass:suff-4} by definition, we could simply choose $\psi=\varphi_K^2$. 
For example, for any vector $v\in (1/2)B_2^d$ and $K=B_2^d-v$,
it is easily verified that $\varphi_K^2$ is twice differentiable with vanishing gradient at zero \ref{ass:suff-1} and $\varphi_K^2$ is $2/(1+\norm{v}_2)^2$-strongly convex \ref{ass:suff-3}. Notably, in this example $\varphi_K$ is not a norm, as the convex body is not centrally symmetric.

However, in many interesting cases the squared Minkowski functional $\varphi_K^2$ is not smooth or not strongly convex (for example, the $\ell_1$-norm), and we need to approximate it with a different function. This is possible for all $K$, as we now show in \cref{lem:existence}. Specifically, we can smoothen and ``strongly convexify'' $\varphi_K^2$: To that end, we denote the Moreau envelope \citep{Moreau1965proximite} of a closed and proper convex function $f$ with $\lambda>0$ as
\begin{equation*}
    (\cM_{\lambda} f)(\alpha)=\inf_{\alpha'\in\RR^d}\mathset{f(\alpha')+\frac{1}{2\lambda}\norm{\alpha-\alpha'}_2^2}.
\end{equation*}

\begin{lemma}
    \label{lem:existence}
    For any convex body $K\subset \RR^d$ that contains the origin in its interior, there exists a potential $\psi$ that satisfies both the continuous and discrete-time versions of \cref{ass:sufficient-conditions-psi} with approximation constant $c_a=4$ and some $\rho>0$. 
    Furthermore, for $\rho=2/(\max_{\alpha\in K}\norm{\alpha}_2^2)$ and sufficiently small $\lambda >0$ independent of $\tau$, the potential
    \begin{equation*}
        \psi(\alpha)=(\cM_{\lambda }\varphi_K^2)(\alpha) + \frac{\rho}{2}\norm{\alpha}_2^2 
    \end{equation*}
    satisfies the discrete-time version of \cref{ass:sufficient-conditions-psi} with approximation constant $c_a = 4$.
\end{lemma}
\noindent We prove \cref{lem:existence} in \cref{subsec:proof-existence} using results from \cite{Planiden2019proximal}. We note that in \cref{lem:existence}, the potentials do not require any knowledge of the radius. Beyond Moreau-smoothing, other smoothing methods, such as the Polar envelope \citep{Friedlander2019polar},
or infimal convolution smoothing \citep{Beck2012smoothing} could be viable options instead. Moreover, there are applications where other potentials that are tailored to the convex body may be more suitable, as we will see in \cref{sec:applications}.

\subsection{A Localized Gaussian Width Risk Bound}
By \cref{lem:existence}, \cref{ass:sufficient-conditions-psi} is always non-vacuous, which leads us to the following theorem; our main result.
\begin{theorem}
\label{thm:in-sample-risk-bound}
    Let $\alpha_0=0\in\RR^d$ and let $\mathset{\alphat}_{t\geq 0}$ be the continuous or discrete-time mirror descent updates on $\Rhat$ using some $\psi$ that satisfies \cref{ass:sufficient-conditions-psi}.
    In the discrete-time case, let $\psi$ be $\rho$-strongly convex and $\Rhat$ be $\beta$-smooth with respect to the same norm, and let the step-size satisfy
    $\eta\leq \frac{\rho}{\beta}\wedge\frac{D_{\psi}(\alphastar,0)}{2}$.
    Then, for any $\eps>0$ and $T:=c_u^2\tau^2 /\eps$ in continuous time and $T:= \lceil 2c_u^2\tau^2 /(\eps\eta)\rceil$
    in discrete time,
    \begin{equation}
    \label{eq:in-sample-risk-bound}
        \min_{0\leq t\leq T}\cR(\alphat)\leq  \frac{2\rzero^2(\X\alphastar,\X K_{3c_a\tau})}{n}+\frac{4\log(1/\delta)}{n}+\eps
    \end{equation}
    with probability at least $1-\exp(-0.1 n)-\delta$ over draws of the noise $\xi$. Moreover, in continuous time, we have that
    \begin{equation}
    \label{eq:in-sample-risk-bound-expectation}
        \EE_\xi\br{\min_{0\leq t\leq T} \cR(\alphat)}\leq  \frac{2\rzero^2(\X\alphastar,\X K_{3c_a\tau})}{n}+\frac{4}{n}+\eps.
    \end{equation}
\end{theorem}
\noindent The proof can be found in \cref{subsec:proof-in-sample-risk-bound} and is outlined in \cref{subsec:proof-outline}. Throughout this paper, we choose $\eps>0$ to balance the right-hand side of \eqref{eq:in-sample-risk-bound}, respectively \eqref{eq:in-sample-risk-bound-expectation}, and we denote the oracle optimal stopping time as
\begin{equation*}
    \tstar:= \argmin_{0\leq t\leq T}\cR(\alphat).
\end{equation*}
We would like to stress that this stopping time can depend on the noise and the ground truth, and hence is not necessarily computable. However, $\tstar\leq T$ quantifies the maximal number of iterations necessary to achieve the statistical complexity. Note that in the case of discrete time, the strong convexity parameter does not influence the bound in \eqref{eq:in-sample-risk-bound}, however, it impacts the bound on the stopping time.

\begin{remark}
\label{rem:bound-rank}
It is easily shown (\cref{subsec:proof-bound-rank}) that for any convex body we can bound the stationary radius as
\begin{equation}
\label{eq:bound-rank}
    \rzero^2(\X K) \leq 4\rank{\X} \leq 4\min\mathset{n,d},
\end{equation}
where $\rank{\X}$ denotes the rank of $\X$.
This is unsurprising since the unconstrained LSE
is known to achieve the rate $\rank{\X}/n$, which is the minimax risk without any shape constraints over $\RR^d$ \citep[Example 15.14]{Wainwright2019}.
\end{remark}

\subsection{A Few Consequences}

\paragraph{Comparison with LSE.} \cref{thm:in-sample-risk-bound} immediately leads us to the following corollary connecting the in-sample risks of ESMD and the LSE, based on the results from \cite{Chatterjee2014}. For this, we must restrict ourselves to cases in which the following assumption holds. 
\begin{assumption}\label{ass:unreal}
    There exists a constant $\cC\geq 1$ such that 
    \begin{equation*}
        \rzero^2(\X\Kt)\leq \cC \cdot \rstar^2(\X\Kt),
    \end{equation*}
    implying $\rzero^2(\X\Kt)\asymp \rstar^2(\X\Kt)$ by \eqref{eq:radii-comparison}. 
\end{assumption}
\cref{ass:unreal} is not very strong: it holds for regular Donsker-type classes, and many non-parametric classes (cf.\ \citep{vandegeer2000empirical}). But, importantly, it does not always hold; see \citet[Sec. 3.1.3]{Prasadan2024some} for a counterexample. Other notable examples of when it does not hold, appear in \cite{kur2023variance}, and see also  \cite{aolaritei2025revisiting}.
We prove the following \cref{cor:comparison-LSE} in \cref{subsec:proof-comparison-LSE}. 
\begin{corollary}
\label{cor:comparison-LSE}
    In the setting of \cref{thm:in-sample-risk-bound}, if $\rstar(\X\Kt)\geq (644/3)^2$ and \cref{ass:unreal} holds, it holds that
    \begin{equation}
    \label{eq:sup-risk-comparison}
        \sup_{\alphastar\in \Kt} \EE_\xi\br{\cR(\alphatstar)}  \leq 84 \cC c_a \cdot \sup_{\alphastar\in \Kt} \EE_\xi\br{\cR(\haLSE)}
    \end{equation}
    for continuous-time ESMD with large enough $T$.
\end{corollary}
\noindent It follows that if $\cC \cdot c_a\lesssim 1$ and the LSE is minimax optimal, then ESMD with optimal stopping time is also minimax optimal. Whether this is the case depends highly on the convex body and the design matrix \citep{Raskutti2011,Kur2020suboptimality}. In \cref{thm:in-sample-risk-bound,cor:comparison-LSE} we did not optimize the constants, and tighter bounds (in terms of constants) could be derived using our arguments.

\paragraph{Minimax optimality.} We can also directly derive a sufficient condition for minimax optimality.
To that end, we define the maximum local entropy (e.g., from \cite{Neykov2022minimax}). Let $M(r,K)$ denote the packing number of $K$ in $\ell_2$-norm at radius $r$, and let $c^\star >0$ be a sufficiently large absolute constant. The local entropy of a set $K\subset \RR^d$ is defined as
\begin{equation*}
    \locEntropy(r,K)=\sup_{\alpha\in K} M(r/c^\star ,(K-\alpha)\cap rB_2^d).
\end{equation*}
We get a sufficient condition for minimax optimality.
\begin{assumption}
    \label{ass:local-entropy}
    It holds for all $r\lesssim \diam(\X \Kt)$ that
    \begin{equation*}
        \sup_{\alpha\in \Kt} \frac{w(\X(\Kt-\alpha) \cap rB_2^n)}{r} \leq \sqrt{\log \locEntropy(r,\X \Kt)}.
    \end{equation*}
\end{assumption}
Examples that satisfy this condition are Donsker-type classes and set constrained models in general dimensions, see \citet{han2021set,kur2019optimality} for discussions.
\begin{corollary}
\label{cor:minimax-optimality}
    Consider the setting of \cref{thm:in-sample-risk-bound}. If \cref{ass:local-entropy} holds and $c_a\lesssim 1$, then continuous-time ESMD is minimax optimal (up to constant factors), that is,
    \begin{equation*}
        \sup_{\alphastar\in \Kt} \EE_\xi \br{\cR(\alphatstar)} \lesssim \inf_{\ha} \sup_{\alphastar\in \Kt} \EE_\xi\br{\cR(\ha)}.
    \end{equation*}
\end{corollary}
We prove \cref{cor:minimax-optimality} in \cref{subsec:proof-minimax-optimality} following \citet[Cor. 2.6]{Prasadan2024some}. \cref{cor:minimax-optimality} yields, for example, that if $\X$ is the identity and $\Kt$ is an $\ell_1$- or $\ell_2$-ball of arbitrary radius $\tau>0$, ESMD is minimax rate optimal. We revisit the $\ell_1$-constrained setting in \cref{subsec:l1-constraints}.

\paragraph{Estimation.} Finally, we would like to highlight that \cref{thm:in-sample-risk-bound} can easily be used to derive bounds on the estimation risk
whenever the design matrix has a vanishing kernel width, cf. \cite{Raskutti2011}.
A matrix $\X\in\RR^{n\times d}$ has a vanishing kernel width with respect to $f$ and $K$, if for all $\alpha\in K-K$
\begin{equation}
\label{eq:vanishing-kernel-width}
    \frac{1}{n}\norm{\X\alpha}_2^2\geq \norm{\alpha}_2^2-f(K,n).
\end{equation}
\begin{corollary}
\label{cor:estimation-risk}
    Under the conditions of \cref{thm:in-sample-risk-bound} and if $\X$ has vanishing kernel width \eqref{eq:vanishing-kernel-width} with respect to $3c_a\tau K$ and $f(3c_a\tau K,n)\lesssim \rzero^2(\X K_{3c_a\tau })/n$,
    the estimation error of ESMD is bounded for all $\alphastar\in \Kt$ as
    \begin{equation*}
        \norm{\alphatstar-\alphastar}_2^2 \lesssim  c_a\frac{ \rzero^2(\X \Kt)}{n}+\frac{\log(1/\delta)}{n}
    \end{equation*}
    with probability at least $1-\exp(-0.1n)-\delta$ over the noise.
\end{corollary}
\noindent We prove \cref{cor:estimation-risk} in \cref{subsec:proof-estimation-risk}. The additional assumption of vanishing kernel width is necessary because parameter estimation is ill-posed if the data matrix does not satisfy any regularity assumptions, especially in the high-dimensional regime where $d>n$.

\subsection{Proof Outline of Theorem \ref{thm:in-sample-risk-bound}}
\label{subsec:proof-outline}
Finally, we provide a short proof outline of \cref{thm:in-sample-risk-bound}; the full proof is in \cref{subsec:proof-in-sample-risk-bound}.
The first ingredient of the proof of \cref{thm:in-sample-risk-bound} is to show that, under \cref{ass:sufficient-conditions-psi}, even without strong convexity of the potential, we have the following inclusion (see \cref{fig:M-convex-hull} for a visualization in the case of an $M$-convex hull)
\begin{equation}
\label{eq:inclusion}
    \forall \alphastar\in \Kt: \quad B_\psi(\alphastar,2D_\psi(\alphastar,0)) \subset 3c_a \Kt.
\end{equation}
\begin{wrapfigure}{r}{0.5\textwidth}
\vspace{-0.6cm}
\centering
\includegraphics[width=\linewidth]{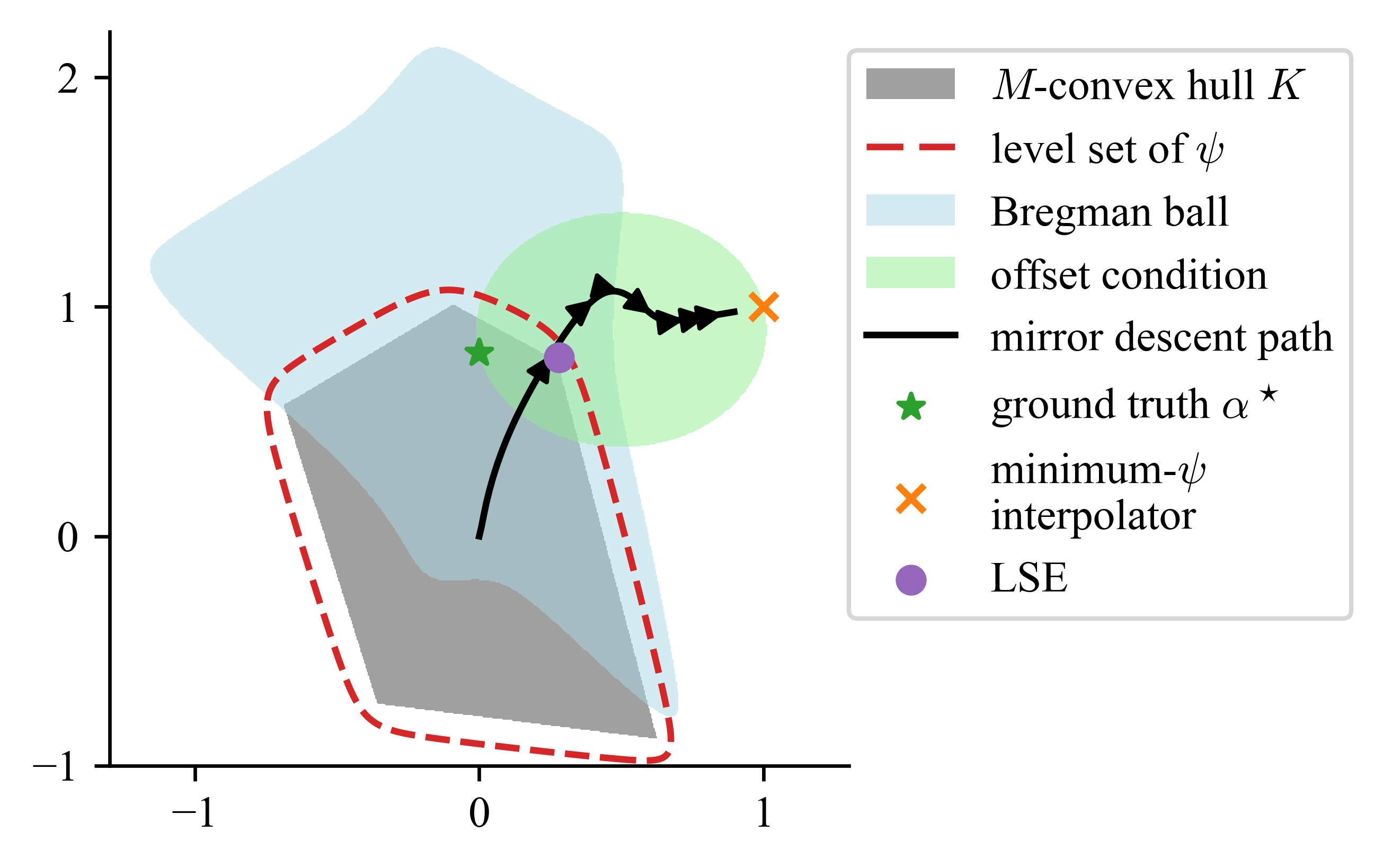}
\caption{{We plot an $M$-convex hull, a level set of the potential from \eqref{eq:M-convex-hull-potential} with $\gamma=10$ and $\rho = 0.2$, the Bregman ball from \eqref{eq:inclusion}, the set of points satisfying the offset condition \eqref{eq:offset-condition}, and the mirror descent path. 
}}
\label{fig:M-convex-hull}
\vspace{-0.8cm}
\end{wrapfigure}
This is useful, as \cite{Kanade2023earlystopped} showed that optimally early-stopped mirror descent is contained in this Bregman ball while satisfying the so-called offset condition \citep{Liang2015learning}, which is defined as
\begin{equation}
\label{eq:offset-condition}
    \Rhat(\alphatstar)-\Rhat(\alphastar)+\cR(\alphatstar) \leq \eps.
\end{equation}
The key step is then to show that using \eqref{eq:inclusion}, we can relate the offset condition \eqref{eq:offset-condition} to the stationary radius from \cref{def:critical-radius} using localization arguments akin to those in \cite{Bellec2016sharporacleinequalitiessquares}. Specifically, we show that \eqref{eq:offset-condition} and \eqref{eq:inclusion} imply that we can bound the in-sample risk with the supremum of a Gaussian process, that is,
\begin{align*}
    &\cR(\alphatstar) \leq \frac{1}{n} \pr{\frac{Z_{\rzero}}{\rzero}}^2+\eps \\ &\text{with} \quad 
     Z_{\rzero} = \sup_{\theta\in (\X(K_{3c_a\tau}-\alphastar)\cap \rzero B_2^n)} \inner{\xi}{\theta},
\end{align*}
where we denote $\rzero=\rzero(\X\alphastar,\X K_{3c_a\tau })$.
We can bound $Z_{\rzero}$ using the concentration of the supremum of a Gaussian process to its expectation $\EE\br{Z_{\rzero}}$. By definition of the stationary radius, we have that $\EE\br{Z_{\rzero}}\leq \rzero^2/2$, which we can plug in and putting things together yields \eqref{eq:in-sample-risk-bound}.

\section{Applications}
\label{sec:applications}
So far, we have only discussed the general case of arbitrary convex bodies and design matrices, which yields the full generality of our main results \cref{thm:in-sample-risk-bound,cor:comparison-LSE,cor:minimax-optimality,cor:estimation-risk}. 
These results allow us to view \cref{ass:sufficient-conditions-psi} as a blueprint. For a given convex body, one can construct a potential that satisfies \ref{ass:suff-1}-\ref{ass:suff-4}. Using this potential, ESMD then enjoys the guarantees of \cref{thm:in-sample-risk-bound}.
We now demonstrate this approach for specific choices of convex bodies and assumptions on the design matrices. We analyze existing potentials and explicitly construct novel potentials.

\paragraph{Assumptions on the Design Matrix.}
For our main results, we made no assumptions on the design matrix. Now we consider two special cases in the applications. The first is when the design matrix is fixed and we only assume normalized columns.
$\X\in\RR^{n\times d}$ is said to be column-normalized, if its columns $\X_i,i\in[d]$ satisfy
\begin{equation}
\label{eq:column-normalization}
    \norm{\X_i}_2\leq \sqrt{n}.
\end{equation}
Note that normalizing to $\sqrt{n}$ here is somewhat arbitrary, and can readily be changed.
The second setting is that of Gaussian design:
$\X$ is Gaussian, if the entries $\X_{ij},i\in[n],j\in[d]$ of $\X$ are i.i.d.\ standard Gaussian, that is, $\X_{ij}\sim \cN(0,1)$. 
Up to constants, Gaussianity implies having normalized columns \eqref{eq:column-normalization} and vanishing kernel width \eqref{eq:vanishing-kernel-width} for some $K$ with high probability; see \cite[Sec. 3.2]{Raskutti2011}.

\subsection{\texorpdfstring{$\ell_p$}{lp}-norms with \texorpdfstring{$p\in(1,2)$}{p in (1,2)}}

The following result shows that $\ell_p$-norms with $p\in(1,2)$ are regular enough such that we can simply choose $\psi$ as the squared Minkowsi functional, i.e., as $\norm{\cdot}_p^2$. 
The mirror descent algorithm associated with this potential is known as the $p$-norm algorithm \citep{Shalev2007online,Levy2019necessary,Orabona2023modern} and can be implemented very efficiently \citep{Gentile2003}. 
Orthogonally, it is worth pointing out that in the interpolation regime, when mirror descent with $\norm{\cdot}_p^2$ is \emph{not} early-stopped, it converges to the minimum $\ell_p$-norm interpolator \citep{Gunasekar2018}. These predictors have been studied in the benign overfitting literature \citep{Donhauser2022,Kur2024minimum}. 
We now derive the rates from \eqref{eq:in-sample-risk-bound} explicitly for column-normalized and Gaussian design.

\begin{proposition}
\label{prop:lp-norm-squared}
    Let $p\in (1,2)$, $1/p+1/q =1$ and $K=B_p^d$. Then $\psi(\alpha)=\norm{\alpha}_p^2$ satisfies the discrete-time version of \cref{ass:sufficient-conditions-psi} where the potential is $\rho=2(p-1)$-strongly convex with respect to the $\ell_p$-norm, and $c_a=1$. If $\X$ is column normalized \eqref{eq:column-normalization}, optimally early-stopped mirror descent achieves for all $\alphastar\in \tau B_p^d$
    \begin{equation*}
        \cR(\alphatstar)\lesssim \frac{\rank{\X}}{n}\wedge 
        \frac{\tau}{\sqrt{n}}
        \begin{cases}
            \sqrt{\log d} & \text{if } p\leq 1+\frac{1}{\log d},\\
            \sqrt{q} d^{1/q} & \text{if } p>1+\frac{1}{\log d}
        \end{cases}
    \end{equation*}
    with probability at least $0.99-\exp(-0.1n)$ over the noise.
    If $\X$ is Gaussian, then for all $\alphastar\in \tau B_p^d$, ESMD achieves the same bound with $\rank{\X}/n=1$ and 
    with probability at least $0.99-2\exp\pr{-0.1n}$ jointly over draws of the design matrix $\X$ and the noise $\xi$.
\end{proposition}
\noindent We prove \cref{prop:lp-norm-squared} in \cref{subsec:proof-lp-norm-squared}. The strong convexity was proved, e.g., in \cite{Shalev2007online,ball1994sharp}.
A tighter bound may be possible using the localized Gaussian width bound from \cite{Gordon2007}.

\paragraph{Minimax (Sub-)Optimality.}
When $\X$ is a scaled identity, the LSE is known to be \emph{sub}-optimal under $\ell_p$-norm constraints \citep{donoho1994minimax,johnstone2017gaussian,Prasadan2024some,aolaritei2025revisiting}, and so by \cref{eq:radii-comparison} we cannot generally hope for our bounds to prove minimax optimality of ESMD in that case.

Surprisingly, however, there seems to be no work (explicitly) establishing the minimax rate under the \emph{worst-case} fixed and Gaussian design. We now show that the rate from \cref{prop:lp-norm-squared} is optimal (up to $p$-dependent-factor) for a worst-case fixed design matrix that is column-normalized.
To that end, we explicitly construct a column-normalized data matrix as a hard instance. This is similar in spirit to a line of research investigating particularly hard design matrices  
\citep{Rigollet2010exponentialscreeningoptimalrates,Zhang2017,Pathak2024design,Foygel2011fastrate,Dalalyan2017}.
At the same time, we show that the rate from \cref{prop:lp-norm-squared} is sub-optimal for the Gaussian design.
\begin{theorem}
\label{thm:minimax-rate-1-2}
    Let $p\in[1+1/\log d,2)$, $1/p+1/q=1$ and let $\tau=1$ for simplicity. Assume that $n^{p/2} \leq d \leq n^{q/2}$.
    There exists a design matrix $\X\in\RR^{n\times d}$ that is column normalized \eqref{eq:column-normalization}, such that the minimax in-sample risk satisfies
    \begin{equation*}
         \inf_{\ha}\sup_{\alphastar\in B_p^d}\EE_\xi \br{\cR(\ha)} \geq 1\wedge c_p \frac{d^{1/q}}{\sqrt{n}},
    \end{equation*}
    where the infimum is taken over all estimators and $c_p>0$ is a constant that may depend on $p$.
    If $\X$ is Gaussian, let $ n(\log d)^{c_0} \leq d \leq n^{q/2}$ and $p\in[ 1+c_1/\log \log d,2)$ for some universal constant $c_0,c_1>0$. Then
    \begin{equation*}
        \inf_{\ha} \sup_{\alphastar\in B_p^d}\EE_\xi \br{\cR(\ha)} \asymp n^{p/2-1}(\log d)^{1-p/2}\vee \frac{d^{2/q}}{n}
    \end{equation*}
    with probability at least $1-c_2\exp(-c_3n)$ over draws of the matrix $\X$, where $c_2,c_3>0$ are some constants.
\end{theorem}

\noindent We prove \cref{thm:minimax-rate-1-2} in \cref{subsec:proof-minimax-rate-1-2}. Note that the upper bound from \cref{prop:lp-norm-squared} and the lower bound from \cref{thm:minimax-rate-1-2} essentially match for the fixed data matrix. 
In \cref{subsec:simulation-minimax-rate}, we present simulations of the LSE on the adversarial data matrix of the first lower bound, showing that it exhibits the rate $d^{1/q}/\sqrt{n}$.

\subsection{\texorpdfstring{$\ell_1$}{l1}-norm}
\label{subsec:l1-constraints}

\begin{table*}
    \caption{ Mirror descent potentials for linear regression over the $\ell_1$-norm ball with which early-stopping is minimax optimal (\cref{thm:l1-all-combined}). Note that the Moreau envelope is \emph{not} the same as in \cref{lem:existence}, which would also be a valid potential for \cref{thm:l1-all-combined}. Here we show $(\cM_\lambda \norm{\cdot}_1+\frac{d\lambda}{2})^2+\frac{\rho}{2}\norm{\cdot}_2^2$, because it has a closed-form solution. See \cref{fig:optimization-paths} in \cref{subsec:risk-along-optimization-path} for example optimization paths.}
    \label{tab:l1-potentials}
    \centering
    \resizebox{\textwidth}{!}{
    \begin{tabular}{l l l l}
        \toprule
         \textbf{name} & \textbf{potential} $\psi(\alpha)$ & \makecell[l]{\textbf{strong convexity} \\ \textbf{parameter} $\rho$} & \makecell[l]{\textbf{approximation} \\ \textbf{constant} $c_a$} \\
        \midrule
        \makecell[l]{squared \\ $\ell_p$-norm} & $\norm{\alpha}_p^2$  with  $1<p\leq 1+1/\log(d)$ & \makecell[l]{$\rho= 2(p-1)\leq \frac{2}{\log d}$  \\ with $\ell_p$-norm} & $c_a=e$ \\[2ex]
        \makecell[l]{Moreau envelope \\ (Huber loss) } & \makecell[l]{$(\sum_{i=1}^d h (\alpha_i) + \tau)^2 + \norm{\alpha}_2^2$ \\
        where $h(x)=\begin{cases} \frac{x^2 d}{4\tau} & \abs{x}\leq \frac{2\tau}{d} \\ \abs{x}-\frac{\tau}{d} & \abs{x}> \frac{2\tau}{d} \end{cases}$} & \makecell[l]{$\rho = 2 $ \\ with $\ell_2$-norm} & $c_a=\sqrt{5}$ \\[5ex]
        \makecell[l]{adjusted \\ hypentropy} & \makecell[l]{$\pr{\sum_{i=1}^d \frac{\pr{ \alpha_i\arcsinh\pr{\alpha_i/\gamma}-\sqrt{\alpha_i^2+\gamma^2}+\gamma+1}}{\arcsinh\pr{\gamma^{-1}}}}^2$ \\ with $\gamma\leq  \sinh(d/\tau)^{-1}\wedge (4\tau)^{-1}\wedge  2^{-1/2}$} & \makecell[l]{$\rho = \arcsinh(\gamma^{-1})^{-2}$ \\ $\leq (\tau/d)^2$ with $\ell_2$-norm}  & $c_a=3$ \\[5ex]
        \makecell[l]{sigmoidal} & \makecell[l]{$\pr{\sum_{i=1}^d \frac{\pr{\log\pr{1+\exp\pr{-\gamma \alpha_i}}+\log\pr{1+\exp\pr{\gamma \alpha_i}}}}{\gamma}}^2$ \\ with $\gamma \geq d\log (4)/\tau$} & \makecell[l]{only valid for \\ continuous time} & $c_a = 2$ \\
        \bottomrule
    \end{tabular}
    }
\end{table*}

When the convex body is an $\ell_1$-norm ball, the corresponding LSE is the LASSO estimator in its constrained form \citep{Tibshirani1996lasso}. It is known \citep[Thm. 7]{Bellec2017localizedgaussianwidthmconvex} that for all column-normalized design matrices \eqref{eq:column-normalization}, if $d\geq \tau \sqrt{n}$, the LSE achieves with high probability over the noise
\begin{equation}
\label{eq:l1-LSE-bound}
     \cR(\haLSE)\lesssim \frac{\rank{\X}}{n} \wedge \tau\sqrt{\frac{\log (ed/(\tau \sqrt{n}))}{n}},
\end{equation}
and there exists at least one column-normalized data matrix for which this is minimax optimal \citep[Thm. 5.3 and Eq. 5.25]{Rigollet2010exponentialscreeningoptimalrates}. 
If $d/(\tau\sqrt{n})\asymp d^\kappa$ with some constant $\kappa>0$ and the rank of the design matrix is large enough, the bound from \eqref{eq:l1-LSE-bound} reduces to the rate $\tau \sqrt{\log(d)/n}$. In \cite[Thm. 3]{Raskutti2011}, it is shown that if this scaling assumption holds and the design matrix has vanishing kernel width \eqref{eq:vanishing-kernel-width} with $f(\tau B_1^d,n)\lesssim \sqrt{\tau}\pr{\log(d)/n}^{1/4}$ (implying it has rank $n$), the minimax lower bound matches the simpler bound.
However, without the scaling assumption and if $\tau\sqrt{n}\geq e$, the bound from \eqref{eq:l1-LSE-bound} is a stronger bound. See also \citet[pgs. 16-17]{Rigollet2010exponentialscreeningoptimalrates} for a discussion.

\citet[Thm. 5]{Kanade2023earlystopped} showed that when the design matrix is column-normalized \eqref{eq:column-normalization}, optimally early-stopped continuous-time mirror descent with the hyperbolic entropy from \cite{Ghai2020}, defined as
\begin{equation*}
    \psi(\alpha)=\sum_{i=1}^d  \alpha_i\arcsinh\pr{\alpha_i/\gamma}-\sqrt{\alpha_i^2+\gamma^2},
\end{equation*}
achieves for appropriate $\gamma>0$ with high probability over draws of the noise
\begin{equation}
\label{eq:hyperbolic-entropy-rate}
     \cR(\alphatstar)\lesssim \tau\frac{\log^{3/2}d}{\sqrt{n}}.
\end{equation} 
As one can see, there is a gap from \eqref{eq:hyperbolic-entropy-rate} to \eqref{eq:l1-LSE-bound} of order (at least) $\log d$. The same potential has also been used for the related setting of sparse noisy phase retrieval \citep{Wu2022} and in the analysis of diagonal networks \cite{woodworth2020kernel}. 
And while a batch conversion of online mirror descent with the first potential from \cref{tab:l1-potentials} is known to achieve the rate $\sqrt{\log(d)/n}$ (e.g., \cite{bach2024learning}), as discussed, there is still a gap to \eqref{eq:l1-LSE-bound} and their proof technique does not apply to our definition of mirror descent.
Employing our results, together with results from \cite{Bellec2017localizedgaussianwidthmconvex}, we can improve upon \eqref{eq:hyperbolic-entropy-rate} and fully close the gap between \eqref{eq:l1-LSE-bound} and \eqref{eq:hyperbolic-entropy-rate}. Because $x\mapsto\norm{x}_1^2$ is \emph{not} differentiable, nor strictly convex, we cannot use it as a potential itself. However, we can use \cref{ass:sufficient-conditions-psi} to derive several alternatives. The next theorem provides a joint analysis of the potentials in \cref{tab:l1-potentials}.
\begin{theorem}
\label{thm:l1-all-combined}
    Suppose that $K=B_1^d$ and $d\geq \tau \sqrt{n}$. Then all potentials from \cref{tab:l1-potentials} satisfy \cref{ass:sufficient-conditions-psi} with the specified constants, and if $\X$ is column-normalized \eqref{eq:column-normalization}, optimally stopped mirror descent using any of the potentials from \cref{tab:l1-potentials} achieves for all $\alphastar\in \tau B_1^d$
    \begin{equation*}
        \cR(\alphatstar)\lesssim \pr{\frac{\rank{\X}}{n} \wedge c_a\tau\sqrt{\frac{\log (ed/(\tau\sqrt{n}))}{n}}} +\frac{\log(1/\delta)}{n}
    \end{equation*}
    with probability at least $1-\exp(-0.1n)-\delta$ over the noise. 
    If $\X$ is Gaussian, then optimal early-stopping achieves
    \begin{equation*}
        \cR(\alphatstar) \lesssim 1\wedge c_a\tau \sqrt{\frac{\log d}{n}}
    \end{equation*}
    with probability at least $0.99-2\exp\pr{-0.1n}$ jointly over draws of the design matrix $\X$ and the noise $\xi$.
\end{theorem}
\noindent We prove \cref{thm:l1-all-combined} in \cref{subsec:proof-l1-all-combined}, where we also specify the constants of the first bound.
As discussed above, because when $\X$ is Gaussian it has vanishing kernel width \citep[Proposition 1]{Raskutti2011}, both bounds are minimax optimal under weak scaling assumptions. Therefore, they cannot be improved upon beyond the constants and ESMD is minimax optimal, as is the LSE. 

Notice how the third potential in \cref{tab:l1-potentials} is an adjusted version of the hypentropy potential from \cite{Ghai2020}. With a few changes to the potential, we closed the logarithmic gap from \eqref{eq:hyperbolic-entropy-rate} to \eqref{eq:l1-LSE-bound}; In particular, the key is that we \emph{square} the potential. The sigmoidal example in \cref{tab:l1-potentials} for continuous-time mirror descent, where strong convexity is not required, is a natural smooth approximation of the absolute value function from \cite{Schmidt2007L1approximation}.
Some example paths using potentials from \cref{tab:l1-potentials}, and the risk along the optimization path, are in \cref{subsec:risk-along-optimization-path}.

\paragraph{Computational-statistical trade-off.}
As we can see in \cref{tab:l1-potentials}, the strong convexity parameter $\rho$ varies depending on the potential we use. For example, the strong convexity parameter $2(p-1)\leq 2/\log(d)$ of the squared $\ell_p$-norm vanishes logarithmically as $d\to\infty$. However, the Moreau envelope-based potential showcases that this is not necessary. Importantly, the stopping time from \cref{thm:in-sample-risk-bound} behaves as $T\asymp 1/\rho$, ignoring all other dependencies. Hence, whether (and how fast) the strong convexity parameter vanishes determines the bound on the computational cost of achieving minimax optimality with the given potentials. Thus, the squared $\ell_p$-norm is not necessarily the best candidate.
Generally, we can improve the constant $c_a$ arbitrarily close to $1$ from above by paying in a smaller $\rho$.

\subsection{\texorpdfstring{$M$}{M}-convex Hulls}

Let the convex body be the convex hull of $M$ points $k_i\in \RR^d$ that contains the origin in its interior. Using the points to represent the convex body is known as the $V$-representation, but one can also transform this representation into a so-called $H$-representation, where the convex body is characterized as the set of solutions to the linear inequality $\A \alpha \leq \ind_{m}$ with $\A\in \RR^{m\times d}$ and $\ind_m=(1,\dots,1)^\top\in \RR^m$, where $m\in\NN$ is some finite number that can be bounded as $m\geq d+1$ and in terms of $M$. See \cite{Schrijver1998theory} for more in-depth background. Given this representation, we can equivalently write the convex body as the solution to $\max_{i\in [m]}\inner{a_i}{\alpha}\leq 1$, where $a_i$ is the $i$-th row of $\A$.
Hence, it is easy to see that 
\begin{equation}
\label{eq:M-convex-hull}
     \varphi_K(\alpha)=\max_{i\in [m]}\inner{a_i}{\alpha}.
\end{equation}
Special cases of this are the $\ell_1$-norm and the $\ell_\infty$-norm.
We can approximate this Minkowski functional using a smoothing from \cite[Example 4.5]{Beck2012smoothing}.
\begin{proposition}
\label{prop:M-convex-hulls}
    Suppose that $K$ is an $M$-convex hull that contains the origin its interior, as described above, and that the Minkowski functional is in the form of \eqref{eq:M-convex-hull} with $\sum_{i=1}^m a_i = 0$. Then $\psi$, defined as
    \begin{equation}
    \label{eq:M-convex-hull-potential}
        \psi(\alpha)=\frac{1}{\gamma^2}\log^2\pr{m\sum_{i=1}^m \exp(\gamma \inner{a_i}{\alpha})}+\frac{\rho}{2}\norm{\alpha}_2^2
    \end{equation}
    with $\rho = 2/ \max_{i\in [M]} \norm{k_i}_2^2$ and $\gamma \geq 2\log (m)/\tau$ satisfies the continuous and discrete-time version of \cref{ass:sufficient-conditions-psi} with $\rho$ for the $\ell_2$-norm and $c_a=3$. Moreover, if 
    $$\phi=\max_{i\in[M]} \frac{\norm{\X k_i}_2}{\sqrt{n}}$$ and $M\geq \tau \phi \sqrt{n}$, optimally early-stopped mirror descent with the potential from \eqref{eq:M-convex-hull-potential} achieves for all $\alphastar\in \Kt$
    \begin{equation*}
        \cR(\alphatstar)\lesssim \frac{\rank{\X}}{n}\wedge c_a \tau \phi  \sqrt{\frac{\log\pr{eM/(\tau\phi \sqrt{n})}}{n}}
    \end{equation*}
    with probability at least $0.99-\exp\pr{-0.1n}$ over draws of the noise.
\end{proposition}
\noindent Using the bound on localized Gaussian width of $M$-convex hulls from \cite{Bellec2017localizedgaussianwidthmconvex}, we prove \cref{prop:M-convex-hulls} in \cref{subsec:proof-M-convex-hulls}, where we also specify the constants.
Noticably, when $K$ is an $\ell_1$-norm ball, we recover the bound from \cref{thm:l1-all-combined} with $M=2d$, $m=2^d$ and $\phi = 1$.

\begin{figure*}[t]
    \centering
    \includegraphics[width=\linewidth]{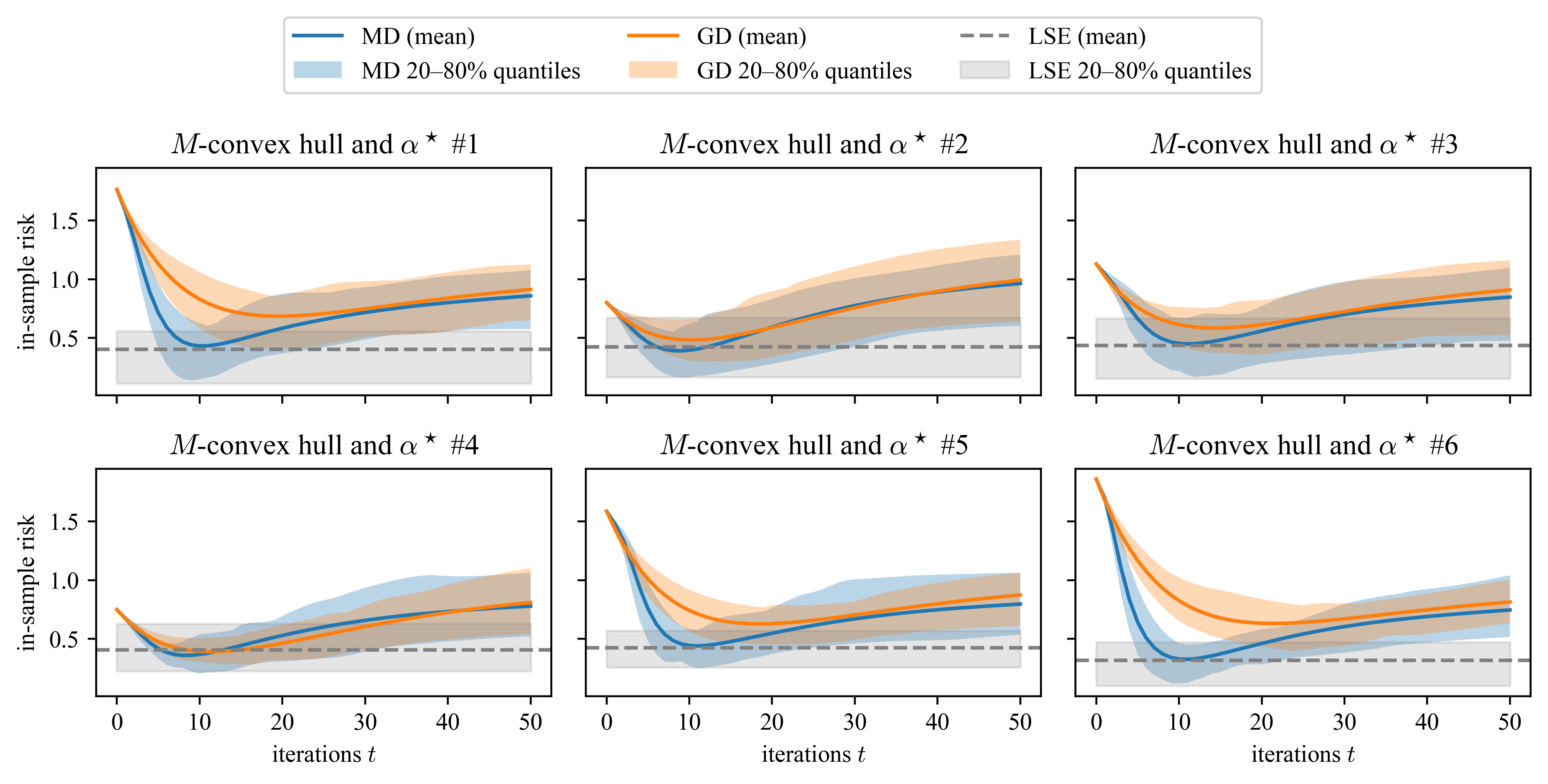}
    \caption{ Risks of mirror descent (MD) with potential from \eqref{eq:M-convex-hull-potential}, gradient descent (GD) and the LSE on six randomly generated convex bodies and ground truths. Using the potential adapts ESMD to the geometry of the convex body.}
    \label{fig:adaptivity-geometry}
\end{figure*}
\paragraph{Simulations.} In \cref{fig:adaptivity-geometry}, we randomly generate six $M$-convex hulls in $\RR^{10}$ with $m=50$ points, and sample a ground truth within each convex body. We let $\X$ be the identity. \cref{fig:adaptivity-geometry} shows how the mirror potential from \cref{eq:M-convex-hull-potential} adapts to the geometry of the convex body and how early-stopped mirror descent performs better than early-stopped gradient descent. We repeat the experiment $30$ times and plot mean, $20$th and $80$th percentile for the instances of the noise. Note that the best iterate has a similar risk to the LSE, as described in \cref{cor:comparison-LSE}.

\section{Discussion}
\label{sec:conclusion}

Our main contribution is to show that the sharp Gaussian width analysis of the LSE from \cite{Bellec2016sharporacleinequalitiessquares} translates to early-stopped mirror descent, when the mirror potential is chosen to approximate the squared Minkowski functional of the convex body.
This enables a general comparison to the LSE, a formulation of sufficient conditions for minimax optimality, and an improvement over the best known bounds for the $\ell_1$-constrained case. Our results extend to general convex constraints and the high-dimensional setting, which had evaded some of the previous analyses.

\begin{wrapfigure}{r}{0.5\textwidth}
    \centering
    \vspace{-0.2cm}
    \includegraphics[width=\linewidth]{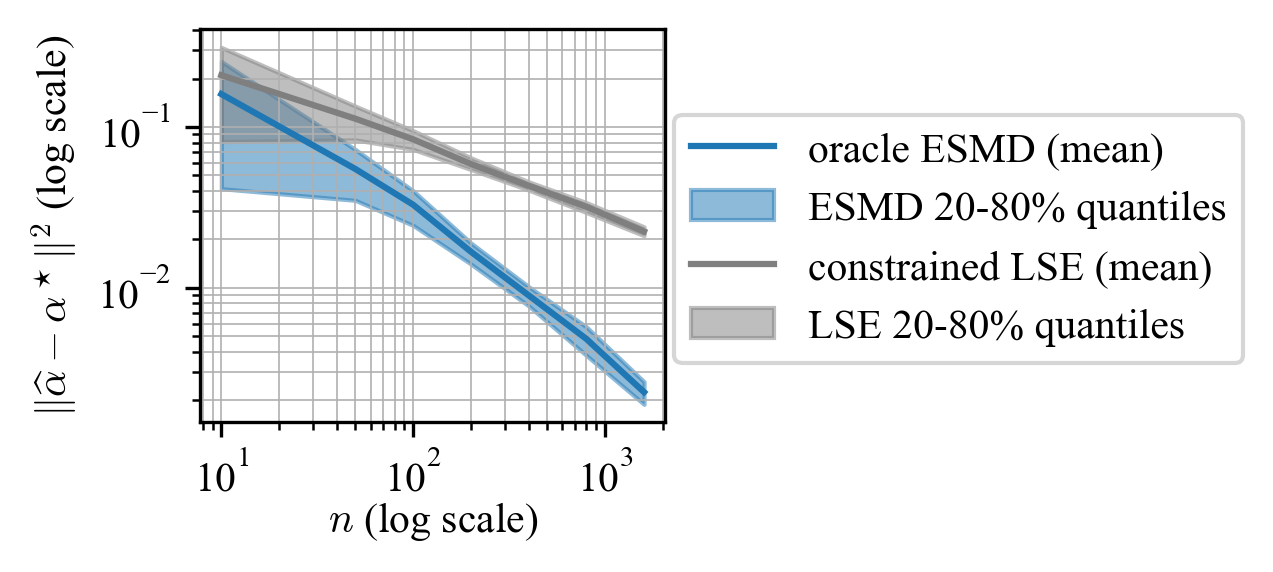}
    \caption{Risk as a funciton of $n$ in the $\ell_{3/2}$-instance described in \cref{sec:conclusion}: ESMD outperforms the LSE.}
    \label{fig:separation-pnorm}
    \vspace{-0.5cm}
\end{wrapfigure}
A caveat of our bound is that it cannot prove that ESMD achieves better rates than the LSE in settings where this may be expected. 
For instance, a setting where the LSE is expected to perform poorly (see Section 2.2 in \citet{aolaritei2025revisiting}) is when we choose $\X=\I_n$, $K=B_{3/2}^d$ and $\alphastar=e_1$, and we scale the noise as $\xi \sim \cN(0,n^{-4/3}\I_n)$. 
Then we can observe in \cref{fig:separation-pnorm} that optimally stopped mirror descent with potential $\norm{\cdot}_{3/2}^2$ outperforms the LSE.
However, a fundamentally different proof technique is required to prove such a discrepancy, as the stationary radius in our bound for ESMD also bounds the risk of the LSE.

Moreover, our comparison between ESMD and the LSE only pertains to the maximal risk over the constraint set, and whether a similar comparison is possible point-wise remains an interesting problem.

Lastly, our analysis is specific to in-sample prediction. The results from \cite{Kanade2023earlystopped} may also be applicable for showing analogous bounds for out-of-sample prediction, but this would require tight bounds on the offset Rademacher complexity of convex bodies and, to facilitate a comparison to the LSE, precise lower bounds for the LSE in out-of-sample prediction. 

\section*{Acknowledgments}
We thank Reese Pathak for helpful feedback and suggesting the experiment in \cref{fig:separation-pnorm}, as well as Tomas Va\v{s}kevi\v{c}ius and Fanny Yang for insightful discussions. Tobias Wegel was partially supported by SNF Grant 204439, and conducted some of this work while he was visiting the Simons Institute for the Theory of Computing. Gil Kur conducted part of this work during his visit to the IDEAL Institute, hosted by Lev Reyzin and supported by NSF ECCS-2217023.
Patrick Rebeschini and Tobias Wegel were partially funded by UK Research and Innovation (UKRI) under the UK government’s Horizon Europe funding guarantee [grant number EP/Y028333/1].

\bibliography{bibliography}

\newpage
\appendix

\counterwithin{lemma}{section}
\addtocontents{toc}{\protect\setcounter{tocdepth}{1}}

\section{Simulations}

\subsection{A Sanity Check for the Hard Design Matrix for \texorpdfstring{$\ell_p$}{lp}-Constraints}
\label{subsec:simulation-minimax-rate}
\begin{wrapfigure}{r}{0.5\textwidth}
    \centering
    \vspace{-0.5cm}
    \includegraphics[width=\linewidth]{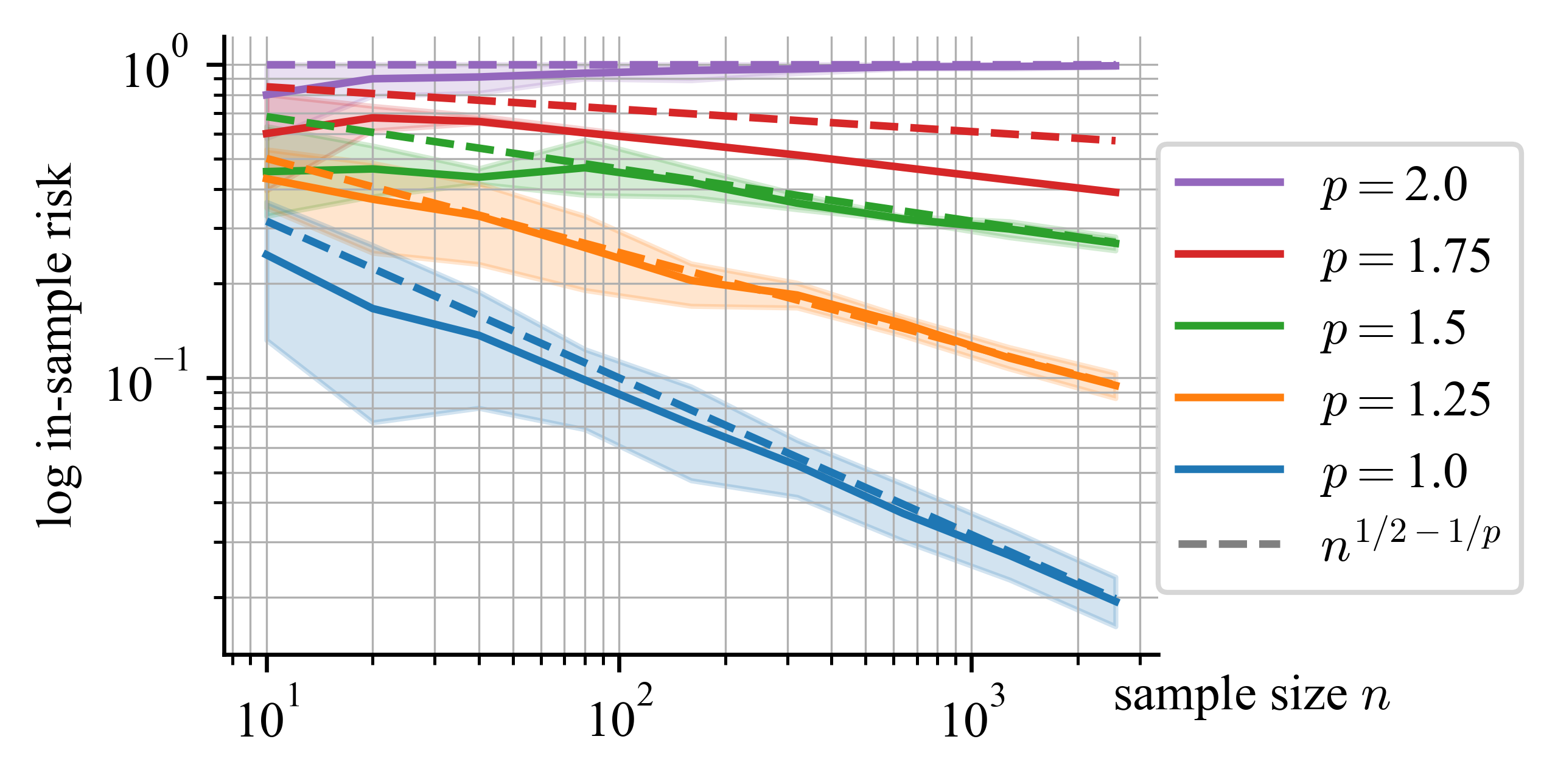}
    \caption{Simulation results for \cref{thm:minimax-rate-1-2}. For each $p$, we plot the log average in-sample risk over 30 experiments in solid, and the log $20$th and $80$th quantile as the shaded regions. The dashed lines show $(1/2-1/p)\log n$.}
    \label{fig:minimax-rate-confirmed}
    \vspace{-0.5cm}
\end{wrapfigure}
In this section, we present a simple simulation using the LSE on the adversarial data matrix constructed in the proof of \cref{thm:minimax-rate-1-2}. We do this as a sanity check of our construction and the rate $d^{1/q}/\sqrt{n}$: Since we know that the LSE also enjoys the upper bound of order $d^{1/q}/\sqrt{n}$ from \cref{prop:lp-norm-squared} (e.g., due to \citet{Bellec2016sharporacleinequalitiessquares}), and according to \cref{thm:minimax-rate-1-2} on the adversarial matrix we have
\begin{equation*}
    \sup_{\alphastar\in B_p^d} \EE_\xi\cR(\haLSE) \gtrsim \frac{d^{1/q}}{\sqrt{n}},
\end{equation*}
there should be a ground-truth $\alphastar\in B_p^d$ attaining the supremum, such that $ \EE_\xi \cR(\haLSE) \asymp  d^{1/q}/\sqrt{n}$.

We do not prove it, but the following simulations suggest that\textemdash at least approximately\textemdash the LSE exhibits the rate $d^{1/q}/\sqrt{n}$ at the ground truth $\alphastar=0$.\footnote{Note that the LSE achieving its worst-case risk at the origin is not obvious a priori \citep{aolaritei2025revisiting}.}
We let $d=n$ vary over $\mathset{10\cdot2^{i}:i\in\mathset{0,\ldots,8}}$, as well as $p$ over $\mathset{1,1.25,1.5,1.75,2}$. We take $\X$ to be the data matrix from the proof of \cref{thm:minimax-rate-1-2}, as defined in \cref{eq:adversarial-matrix} of \cref{subsec:proof-minimax-rate-1-2}, and $\alphastar=0$. 
Thus, the rate from \cref{thm:minimax-rate-1-2} is given by $ \sim d^{1/q}/\sqrt{n}=n^{1-1/p-1/2}= n^{1/2-1/p}$. 

For each $(n,p)$, we then repeat the following experiment $30$ times. 
We sample $y=\X\alphastar+\xi = \xi\sim\cN_n(0,\I_n)$ and 
compute the estimator $\haLSE \in \argmin_{\alpha\in B_p^d} \Rhat(\alpha)$
using the Python library \texttt{cvxpy} \citep{cvxpy}, which is designed for convex optimization.

\cref{fig:minimax-rate-confirmed} shows the results. We plot the sample size $n$ against the average in-sample prediction risk $ \cR(\haLSE)$ in log-scale (solid lines), as well as the (logarithm of) the $20$th and $80$th quantiles (shaded regions) for every pair $(n,p)$. Because $\cR(\haLSE) \sim n^{1/2-1/p}$, we should see a linear dependence with slope $1/2-1/p$. Therefore, we also plot the corresponding lines (dashed lines). And indeed, comparing the risks achieved by the LSE to the rate $n^{1/2-1/p}$, we see that the behavior matches.
This confirms our construction and \cref{thm:minimax-rate-1-2} according to this sanity check.

\subsection{Risk Along \texorpdfstring{$\ell_1$}{l1} Optimization Paths}
\label{subsec:risk-along-optimization-path}

We plot some example optimization paths of some potentials from \cref{subsec:l1-constraints} respectively \cref{tab:l1-potentials} in \cref{fig:optimization-paths}. 
\begin{figure}[h!]
    \centering
    \subfloat[Optimization and regularization paths.]{
    \includegraphics[width=0.63\linewidth]{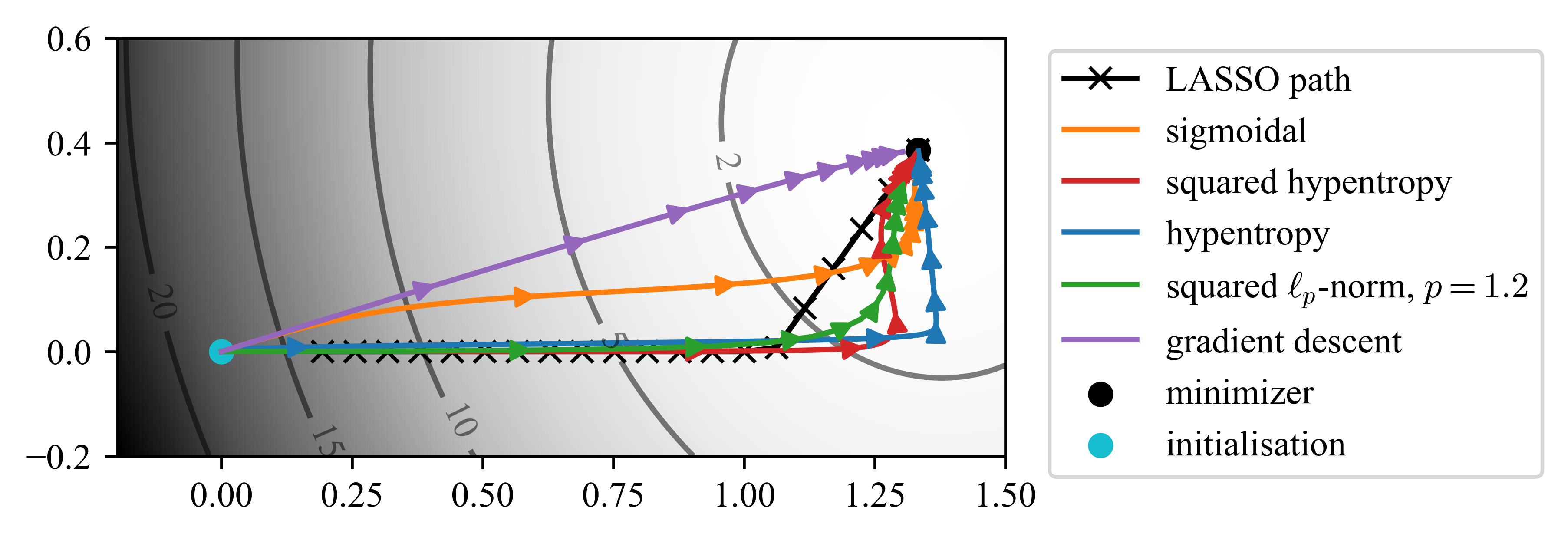}
    \label{fig:iterate-paths}
    }
    \subfloat[Risk along mirror descent paths.]{
    \includegraphics[width=0.35\linewidth]{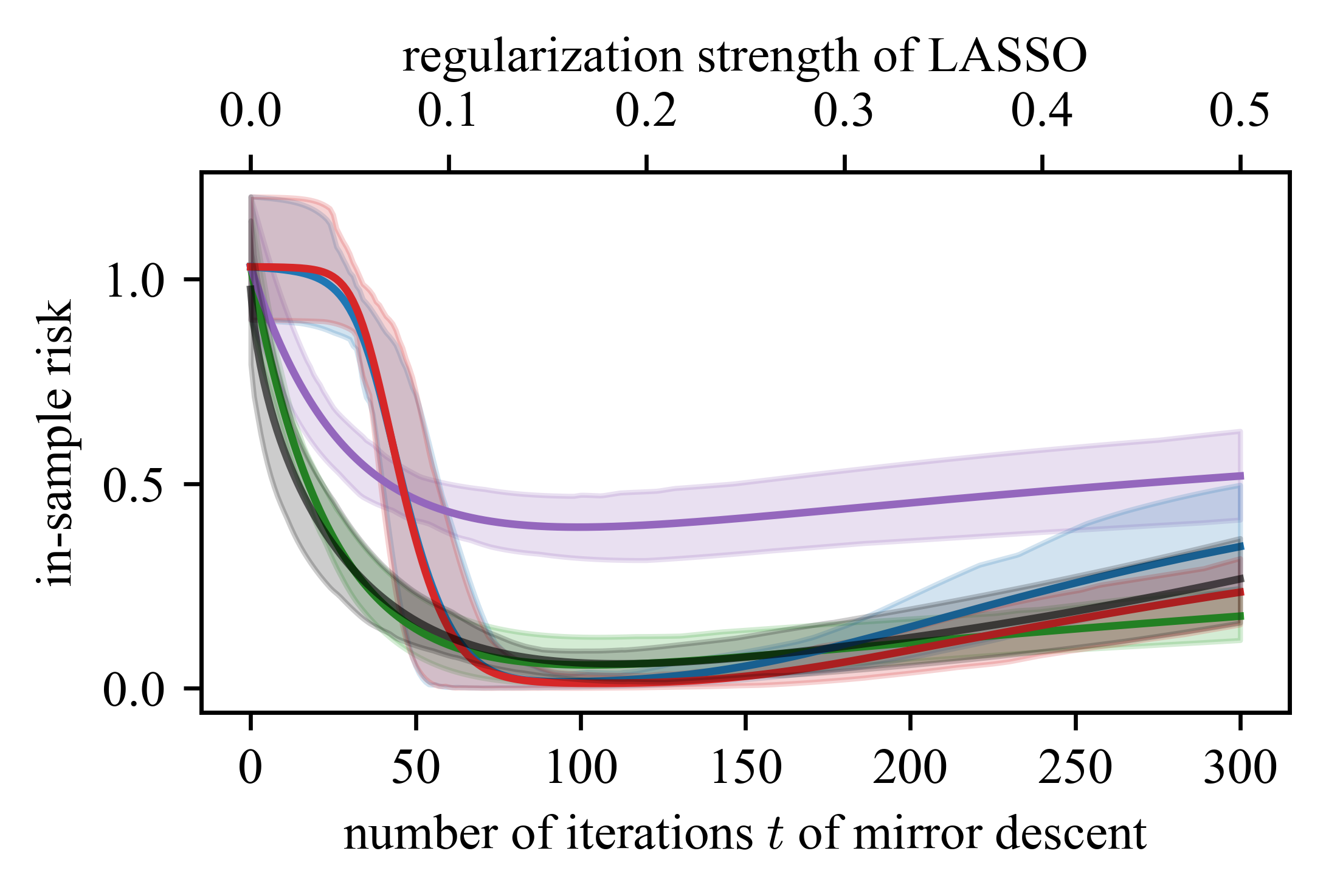}
    \label{fig:risk-along-path}
    }
    \caption{
    Optimization paths of different known and new potentials from this work for regression over $\ell_1$-balls. 
    }
    \label{fig:optimization-paths}
\end{figure}
The squared hypentropy ``fixes'' issues arising for previously used hypentropy from \cite{Ghai2020}. 
In \cref{fig:optimization-paths}\protect\subref{fig:iterate-paths} we plot paths on one data instance for a two-dimensional problem.
Note that the paths can (somewhat) deviate from the LASSO path, but early-stopping still achieves minimax rates (\cref{subsec:l1-constraints}).
In \cref{fig:optimization-paths}\protect\subref{fig:risk-along-path}, we take $\X$ to be Gaussian  for $n=d=100$ and $\alphastar$ to be 1-sparse. We repeat the experiment 50 times and plot mean, $10$th and $90$th percentile.

\subsection{Implementation Details}
Finally, we just note that for all potentials, even those for which the (inverse) gradient does not have a closed-form solution, we can implement mirror descent using the the following expression which is equivalent to \cref{def:discrete-mirror-descent};
\begin{equation*}
    \alpha_{t+1} \in \argmin_{\alpha\in \RR^d} \pr{ \inner{\nabla\Rhat(\alphat)}{\alpha-\alphat}+\frac{1}{\eta}D_\psi(\alpha,\alphat)}.
\end{equation*}
We implement all mirror descent experiments with this approach using \texttt{cvxpy} \cite{cvxpy}. In particular, \cref{fig:M-convex-hull,fig:adaptivity-geometry,fig:separation-pnorm,fig:optimization-paths} are created using mirror descent updates that are implemented in the same way.

\section{Proofs}

\subsection{Preliminaries}

We first show this preliminary result that we use in multiple proofs.
\begin{lemma}
    For any convex body $K$ that contains $0$ in its interior, and any $\tau\geq 1$, it holds that
\begin{equation}
\label{eq:critical-radius-scaling}
    \rzero^2(\tau K):=\sup_{\alphastar\in \tau K}\rzero^2(\alphastar,\tau K)\leq \tau \cdot \sup_{\alphastar\in K} \rzero^2(\alphastar,K) =: \tau \cdot \rzero^2(K).
\end{equation}
\end{lemma}

\begin{proof}
    Note that, by definition, $\rzero(K)$ is bounded by any $r\geq 0$ that satisfies for all $\alphastar \in K$ (cf. \eqref{eq:localized-inequality})
\begin{equation*}
     w\pr{(K-\alphastar)\cap rB_2}\leq \frac{r^2}{2}.
\end{equation*}
Take any $\alphastar \in \tau K$ and define $R = \sqrt{\tau} \cdot \rzero(K)$. We have that
\begin{align*}
    w\pr{(\tau K-\alphastar)\cap RB_2} &= w\pr{\tau \pr{( K-\alphastar/\tau)\cap (R/\tau) B_2}}\\
    &=\tau \cdot w\pr{( K-\alphastar/\tau)\cap (R/\tau) B_2} \\
    &= \tau \cdot w\pr{( K-\alphastar/\tau)\cap (\rzero(K)/\sqrt{\tau}) B_2} \\
    &\leq \tau \cdot w\pr{( K-\alphastar/\tau)\cap \rzero(K) B_2} \\
    &\leq \tau \cdot \frac{\rzero^2(K)}{2} = \frac{R^2}{2}.
\end{align*}
Consequently, $\rzero^2(\alphastar,\tau K)\leq R^2 = \tau \rzero^2(K)$ for every $\alphastar\in \tau K$, and hence we have that $\sup_{\alphastar \in \tau K}\rzero^2(\alphastar, \tau K)\leq \tau \cdot \sup_{\alphastar \in K}\rzero^2(\alphastar, K)$, or in short, $\rzero^2(\tau K) \leq \tau \cdot \rzero^2(K)$.
\end{proof}

\subsection{Proof of Lemma \ref{lem:existence}}
\label{subsec:proof-existence}

    Let $K$ be any convex body with zero contained in its interior. It is known that for every $\delta>0$, there exists a convex body $K_\delta$ such that $K\subset K_\delta\subset (1+\delta)K$, and the corresponding Minkowski functional squared $\varphi_\delta^2 := \varphi_{K_\delta}^2$ is twice continuously differentiable on $\RR^d$ with $\nabla \varphi_\delta^2 (0)=0$,
    see for example \cite[Section 27]{Bonnesen_1934}, \cite{Schneider2013convex}, \cite[Chapter 13]{Guirao2022renormings} and \cite[Chapter 7, \S 9]{Hajek2014smooth}. For this convex body, it clearly holds that $\frac{1}{1+\delta}\varphi_K\leq \varphi_\delta\leq \varphi_K$.

    Furthermore, since $\varphi_\delta^2$ is a convex and twice continuously differentiable function, the function $\psi$ given by
    \begin{equation*}
        \psi(\alpha)=\varphi_\delta^2(\alpha)+\frac{\rho}{2}\norm{\alpha}_2^2
    \end{equation*}
    with $\rho = 2/(\max_{\alpha\in K}\norm{\alpha}_2^2)$
    satisfies
    \paragraph{\ref{ass:suff-1}.} $\psi$ is twice continuously differentiable and the gradient satisfies $\nabla \psi(0) =0$. Therefore, \ref{ass:suff-1} is satisfied.
    
    \paragraph{\ref{ass:suff-2}.} The function $\alpha\mapsto \sqrt{\psi(\alpha)}=\sqrt{\varphi_\delta^2(\alpha)+\frac{\rho}{2}\norm{\alpha}_2^2}$ is convex, as we can write it as \begin{equation*}
        \sqrt{\psi(\alpha)}=\norm{\br{\varphi_\delta(\alpha),\sqrt{\frac{\rho}{2}}\norm{\alpha}_2}^\top}_2
    \end{equation*}
    which is the composition of convex functions and a non-decreasing one, making $\sqrt{\psi}$ convex. Therefore, \ref{ass:suff-2} is satisfied.
        
    \paragraph{\ref{ass:suff-3}.} $\psi$ is $\rho$-strongly convex with respect to the $\ell_2$-norm \ref{ass:suff-3}. This is due to the well-known fact \citep{Nikodem2011characterization} that, because $\norm{\cdot}_2$ is induced by an inner product, $\psi$ is $\rho$-strongly convex with respect to $\norm{\cdot}_2$ if and only if $\psi-\frac{\rho}{2}\norm{\cdot}_2^2=\varphi_\delta^2$ is convex, which clearly holds since $\varphi_\delta$ is a Minkowski functional.

    \paragraph{\ref{ass:suff-4}.} We have that for all $\alpha\in \RR^d$ and $\alpha' \in \tau K$
    \begin{align*}
        \sqrt{\psi(\alpha)}&=\sqrt{\varphi_\delta^2(\alpha)+\frac{\rho}{2}\norm{\alpha}_2^2}\geq \varphi_\delta(\alpha) \geq \frac{1}{1+\delta} \varphi_K(\alpha)=: \frac{1}{c_l} \varphi_K(\alpha) \\
        \sqrt{\psi(\alpha')} &=\sqrt{\varphi_\delta^2(\alpha')+\frac{\rho}{2}\norm{\alpha'}_2^2} \leq \sqrt{\varphi_K^2(\alpha')+\frac{1}{\max_{\alpha\in K}\norm{\alpha}_2^2}\norm{\alpha'}_2^2} \leq \sqrt{2}\tau=: c_u \tau.
    \end{align*}
    Hence \ref{ass:suff-4} is satisfied.

    For discrete time, we only need the potential to be differentiable (not twice differentiable), and hence we can also use the Moreau envelope of $\varphi_K^2$, defined as
    \begin{equation*}
        \cM_{\lambda} \varphi_K^2(\alpha):=\inf_{\alpha'}\mathset{\varphi_K^2(\alpha')+\frac{1}{2\lambda}\norm{\alpha'-\alpha}_2^2}
    \end{equation*}
    and let $\psi(\alpha) = \cM_{\lambda} \varphi_K^2(\alpha) + \frac{\rho}{2}\norm{\alpha}_2^2$ with $\rho=2/(\max_{\alpha\in K}\norm{\alpha}_2^2)$.
    We can verify all necessary properties.

    \paragraph{\ref{ass:suff-1}.} $\psi$ is continuously differentiable with gradient \citep[Theorem 2.26]{Rockafellar2009variational}
        \begin{equation*}
            \nabla \psi(\alpha)=\nabla \cM_\lambda \varphi_K^2(\alpha)+\rho \alpha = \frac{1}{\lambda}(\alpha-P_\lambda \varphi_K^2(\alpha)) +\rho\alpha\quad\text{with}\quad P_\lambda \varphi_K^2(\alpha)=\argmin_{\alpha'}\mathset{\varphi_K^2(\alpha')+\frac{1}{2\lambda}\norm{\alpha'-\alpha}_2^2},
        \end{equation*}
        so that $\nabla \psi (0)=0$, cf.\ \cite[Theorem 5.6]{Planiden2019proximal}, and thus is satisfies \ref{ass:suff-1}.
        
    \paragraph{\ref{ass:suff-2}.} $\sqrt{\psi}$ is convex by \cite[Theorem 5.6]{Planiden2019proximal} and the same argument as above, hence it satisfies \ref{ass:suff-2}.
    
    \paragraph{\ref{ass:suff-3}.} $\psi$ is $\rho$-strongly convex with respect to the $\ell_2$-norm \ref{ass:suff-3} by the same argument as above, noting that the squared Moreau-envelope is convex \citep[Theorem 5.6]{Planiden2019proximal}.
    
    \paragraph{\ref{ass:suff-4}.} To verify \ref{ass:suff-4}, we use the following fact that $\sqrt{\cM_\lambda \varphi_K^2} \leq \varphi_K$, as well as the pointwise convergence $\lim_{\lambda\to 0} \sqrt{\cM_\lambda \varphi_K^2}(\alpha) = \varphi_K(\alpha)$ for all $\alpha\in\RR^d$ \citep[Theorem 5.6]{Planiden2019proximal}. As $\sqrt{\cM_\lambda \varphi_K^2}$ is continuous and $\mathset{\alpha\in\RR^d \setmid \varphi_K(\alpha)=1}$ is compact, by Dini's theorem this pointwise convergence implies uniform convergence, i.e., for the function $\zeta(\lambda)=\min_{\varphi_K(\alpha)=1}\sqrt{\cM_\lambda \varphi_K^2}(\alpha)\leq 1$ that $\lim_{\lambda\to 0}\zeta(\lambda)=1$. Therefore, there exists a $\lambda_0$ such that for all $\lambda \leq \lambda_0$ it holds $\zeta(\lambda)\geq 1/2$ and hence also $\sqrt{\cM_\lambda \varphi_K^2(\alpha)}\geq \varphi_K(\alpha)/2$ for all $\alpha$ with $\varphi_K(\alpha)=1$. Because $\sqrt{\cM_\lambda \varphi_K^2}$ is another Minkowski functional  \citep[Theorem 5.6]{Planiden2019proximal} and hence positive homogeneous, the same holds on $\RR^d$.

        Therefore, for $\rho=2/(\max_{\alpha\in K}\norm{\alpha}_2^2)$ and all $\alpha\in \RR^d$, $\alpha'\in \tau K$ it holds
        \begin{align*}
            \sqrt{\psi(\alpha)}&=\sqrt{\cM_{\lambda} \varphi_K^2(\alpha) + \frac{\rho}{2}\norm{\alpha}_2^2} \geq \sqrt{\cM_{\lambda} \varphi_K^2(\alpha)} \geq \frac{1}{2}\varphi_K(\alpha) = \frac{1}{c_l}\cdot \varphi_K(\alpha) \\
            \sqrt{\psi(\alpha')}&=\sqrt{\cM_{\lambda} \varphi_K^2(\alpha') + \frac{\rho}{2}\norm{\alpha'}_2^2} \leq \sqrt{\varphi_K^2(\alpha')+\frac{1}{\max_{\alpha\in K}\norm{\alpha}_2^2}\norm{\alpha'}_2^2} \leq \sqrt{2}\tau =: c_u \tau
        \end{align*}
    This concludes the proof.

\subsection{Proof of Theorem \ref{thm:in-sample-risk-bound}}
\label{subsec:proof-in-sample-risk-bound}
We begin by proving this auxiliary statement:
Let \cref{ass:sufficient-conditions-psi} hold. Then
\begin{equation}
\label{eq:inclusion-2}
    B_\psi(\alphastar,2 D_\psi(\alphastar,0))\subset K_{3c_a\tau}.
\end{equation}
The proof is based on the following two inequalities. We denote $f=\sqrt{\psi}$ so that $\psi(\alpha)=f^2(\alpha)$. Hence, for $f^2(\alpha)>0$,
\begin{align*}
    2D_\psi(\alphastar,0)&=2\pr{f^2(\alphastar)-f^2(0)-\inner{\nabla f^2(0)}{\alphastar-0}} \\
    &\leq 2f^2(\alphastar) &&(\ref{ass:suff-1}: \nabla f^2(0)=0)\\
    D_\psi(\alphastar,\alpha)&= f^2(\alphastar)-f^2(\alpha)-\inner{\nabla f^2(\alpha)}{\alphastar-\alpha} \\
    &=f^2(\alphastar)-f^2(\alpha)-2f(\alpha)\inner{\nabla f(\alpha)}{\alphastar-\alpha} && (\text{chain rule}) \\
    &\geq f^2(\alphastar)-f^2(\alpha)-2f(\alpha)(f(\alphastar)-f(\alpha)) && (\ref{ass:suff-2}: \text{convexity of $f$}) \\
    &=f^2(\alphastar)+f^2(\alpha)-2f(\alpha)f(\alphastar) \\
    &= (f(\alphastar)-f(\alpha))^2
\end{align*}
Since for every $\alpha\in B_\psi(\alphastar,2D_\psi(\alphastar,0))$ we have by definition that $D_\psi(\alphastar,\alpha)\leq 2D_\psi(\alphastar,0)$, the following inequality holds too;
\begin{equation*}
    (f(\alphastar)-f(\alpha))^2 \leq 2f^2(\alphastar) \implies f(\alpha)\leq (1+\sqrt{2})f(\alphastar) \leq 3f(\alphastar).
\end{equation*}
Therefore, by \ref{ass:suff-4} we have that
\begin{equation*}
    \varphi_K(\alpha) \leq c_lf(\alpha) \leq 3c_lf(\alphastar) \leq 3c_lc_u\tau = 3c_a\tau
\end{equation*}
and hence $B_\psi(\alphastar,2D_\psi(\alphastar,0))\subset 3c_a\Kt$, which finishes the proof of this first auxiliary part. Notice how we could improve the constant $3$ arbitrarily close to $2$ in the above argument.

Define the event
\begin{equation*}
    A_1:= \mathset{\Rhat(\alphastar)\leq 2}.
\end{equation*}
This event holds with high probability, as $\Rhat(\alphastar)=\norm{\xi}_2^2/n$, and thus
\begin{align*}
    \PP\pr{A_1}&=\PP\pr{\norm{\xi}_2^2/n\leq 2} = \PP\pr{\norm{\xi}_2^2\leq 2n} \\
    &\geq 1-\pr{\frac{2n}{n}\exp\pr{1-\frac{2n}{n}}}^{n/2} \\
    &= 1-\exp\pr{\pr{\log(2)+1-2}n/2} \\
    &\geq 1-\exp\pr{-0.1 n}
\end{align*}
where we use Chernoff's bound for $\chi_n^2$-distributed random variables in the first inequality.
Conditioned on $A_1$, we have by \cite[Theorems 2 and 3]{Kanade2023earlystopped} that for
\begin{equation*}
    \eta\leq \frac{\rho}{\beta} \wedge \frac{ D_{\psi}(\alphastar,0)}{2}
\end{equation*}
(the latter of which implies $\eta \Rhat(\alphastar)\leq D_\psi(\alphastar,0)$)
and any $\eps>0$ we may choose later, there exists a stopping time 
\begin{equation*}
    t^\star\leq \left\lceil\frac{D_\psi(\alphastar,0)+\eta \Rhat(\alphastar)}{\eta\eps}\right\rceil \leq \left\lceil\frac{2D_\psi(\alphastar,0)}{\eta\eps}\right\rceil \leq \left\lceil\frac{2c_u^2\tau^2}{\eta\eps}\right\rceil=T \quad \text{or} \quad t^\star\leq \frac{D_\psi(\alphastar,0)}{\eps}\leq \frac{c_u^2\tau^2}{\eps}=T
\end{equation*}
for the discrete and continuous-time cases respectively, such that the following two hold:
\begin{enumerate}
    \item For all $t\leq \tstar$, it holds $\alphat\in B_{\psi}(\alphastar,D_\psi(\alphastar,0)+\eta \Rhat(\alphastar))$. 
    \item $\Rhat(\alphatstar)-\Rhat(\alphastar)+\cR(\alphatstar)\leq \eps$, which is called the offset condition with parameter $\eps>0$.
\end{enumerate}
From the choice of $\eta$, on the event $A_1$, we get $\alphat \in B_{\psi}(\alphastar,2D_\psi(\alphastar,0))$.
As we have shown in the auxiliary part \eqref{eq:inclusion-2}, this implies $\alpha_t\in 3c_a\Kt$ for all $t\leq t^\star$. We now combine this with the second part in a localization argument.

By rearranging the offset condition, we retrieve the inequality
\begin{align}
    \cR(\alphatstar) &\leq \Rhat(\alphastar)-\Rhat(\alphatstar) + \eps \nonumber\\
    &= \frac{1}{n}\norm{y-\X\alphastar}_2^2-\frac{1}{n}\norm{y-\X\alphatstar}_2^2+\eps \nonumber\\
    &= \frac{1}{n}\norm{\xi}_2^2-\frac{1}{n}\norm{\X(\alphastar-\alphatstar)}_2^2-\frac{2}{n}\inner{\xi}{\X(\alphastar-\alphatstar)} -\frac{1}{n}\norm{\xi}_2^2+\eps \nonumber\\
    &=-\frac{1}{n}\norm{\X(\alphastar-\alphatstar)}_2^2-\frac{2}{n}\inner{\xi}{\X(\alphastar-\alphatstar)}+\eps, \nonumber\\
    &=\frac{2}{n}\inner{\xi}{\X(\alphatstar-\alphastar)}-\cR(\alphatstar)+\eps. \label{eq:basic-inequality}
\end{align}
We \emph{could} use this to get a unlocalized uniform bound, by bounding this as $\cR(\alphatstar)\leq \sup_{\alpha\in K_{3c_a\tau}}\frac{1}{n}\inner{\xi}{\X(\alphastar-\alpha)}+\eps/2$. However, we can improve upon this using a localization argument.
We follow an argument akin to \cite[Theorem 2.3]{Bellec2016sharporacleinequalitiessquares} where it was used for the LSE. 
Define the random variable
\begin{equation*}
    Z_r: = \sup_{\theta\in (\X K_{3c_a\tau}-\X \alphastar)\cap rB_2^n} \inner{\xi}{ \theta}
\end{equation*}
which is the supremum of a Gaussian process indexed by $\theta\in (\X K_{3c_a\tau}-\X \alphastar)\cap rB_2^n$. 
By concentration of the supremum of Gaussian processes \citep[Theorem 5.8]{Boucheron2013concentration}, for any $r>0$, the event
\begin{equation*}
    A_2 =\mathset{Z_r \leq \EE\br{Z_r} +r\sqrt{2\log(1/\delta)}}
\end{equation*}
has probability at least $1-\delta$.
Recall \cref{def:critical-radius}, in particular, of the stationary radius
\begin{equation*}
    \rzero(\X\alphastar,\X K_{3c_a\tau}) = \inf \mathset{r\geq 0 \setmid w\pr{(\X K_{3c_a\tau}-\X \alphastar)\cap rB_2^n}-\frac{r^2}{2} \leq 0}
\end{equation*}
and denote $\rzero= \rzero(\X\alphastar,\X K_{3c_a\tau})$. By definition, we have that
\begin{equation*}
    \EE\br{Z_{\rzero}}= w\pr{(\X K_{3c_a\tau}-\X \alphastar)\cap \rzero B_2^n}\leq \frac{\rzero^2}{2}
\end{equation*}
and hence, conditioned on $A_2$ with $r=\rzero$, we have
\begin{equation*}
    Z_{\rzero}\leq \frac{\rzero^2}{2}+\rzero\sqrt{2\log(1/\delta)}.
\end{equation*}

We now make a case distinction between $\cR(\alphatstar)\leq \rzero^2/n$ and $\cR(\alphatstar)> \rzero^2/n$ conditioned on $A_2$.
In the first case, we get trivially 
\begin{equation*}
    \cR(\alphatstar)\leq \frac{\rzero^2}{n} \leq \frac{1}{n}\pr{\rzero+\sqrt{2\log(1/\delta)}}^2+\eps.
\end{equation*}

In the second case when $ \cR(\alphatstar)=\frac{1}{n}\norm{\X(\alphatstar-\alphastar)}_2^2> \rzero^2/n$, we define $\lambda = \rzero/\norm{\X(\alphatstar-\alphastar)}_2\in(0,1)$ and $v=\lambda \X\alphatstar+(1-\lambda)\X\alphastar \in \X K_{3c_a\tau}$. Note $(v-\X\alphastar)/\lambda = \X\alphatstar-\X\alphastar$ as well as $\cR(\alphatstar)=\frac{1}{n}\norm{\X(\alphastar-\alphatstar)}_2^2=\rzero^2/(n\lambda^2)$, and hence bounding \cref{eq:basic-inequality} yields with probability at least $1-\delta$
\begin{align*}
    \frac{2}{n}\inner{\xi}{\X(\alphatstar-\alphastar)}-\cR(\alphatstar)+\eps&=\frac{1}{n}\pr{\frac{2}{\lambda }\inner{\xi}{v-\X\alphastar}-\frac{\rzero^2}{\lambda^2}} +\eps\\
    &\leq \frac{1}{n}\pr{\frac{2}{\lambda}Z_{\rzero}-\frac{\rzero^2}{\lambda^2}}+\eps\\
    &=\frac{1}{n}\pr{\frac{2\rzero}{\lambda}\frac{Z_{\rzero}}{\rzero}-\frac{\rzero^2}{\lambda^2}}+\eps\\
    &\leq \frac{1}{n}\pr{\frac{Z_{\rzero}}{\rzero}}^2+\eps &&(\text{since }2ab-b^2\leq a^2)\\
    &\leq \frac{1}{n}\pr{\frac{\rzero}{2}+\sqrt{2\log(1/\delta)}}^2+\eps.
\end{align*}
Hence, either way, we have that
\begin{equation*}
    \cR(\alphatstar)=\frac{1}{n}\norm{\X(\alphatstar-\alphastar)}_2^2\leq \frac{1}{n}\pr{\rzero+\sqrt{2\log(1/\delta)}}^2+\eps \leq \frac{2\rzero^2+4\log(1/\delta)}{n}+\eps. 
\end{equation*}
Taking the union bound on $A_1$ and $A_2$, we get that with probability at least $1-\exp\pr{-0.1n}-\delta$, it holds
\begin{equation*}
    \cR(\alphatstar)\leq \frac{2\rzero^2}{n}+\frac{4\log(1/\delta)}{n} + \eps
\end{equation*}
where $\rzero = \rzero(\X\alphastar,\X K_{3c_a\tau})$. This concludes the proof of the high probability bound.

The in-expectation bound for continuous time mirror descent follows by observing that we do not require the event $A_1$ to hold, and so the high probability bound holds with probability $1-\delta$ (in fact, one could improve its constants, which we neglect here for simplicity). In particular, denoting $z=\frac{2\rzero^2}{n}+\eps$ we have that
\begin{equation*}
    \cR(\alphatstar)\leq z+\frac{4\log(1/\delta)}{n}.
\end{equation*}
From tail integration, we then obtain
\begin{align*}
    \EE_\xi\br{\cR(\alphatstar)}&\leq \int_0^\infty \PP\pr{\cR(\alphatstar)>x}dx \\
    &=\int_0^{z}\PP\pr{\cR(\alphatstar)>x}dx + \int_{z}^\infty \PP\pr{\cR(\alphatstar)>x}dx \\
    &\leq z + \int_z^\infty \exp\pr{-\frac{n}{4}(x-z)}dx\\
    &= z+\frac{4}{n},
\end{align*}
concluding the in-expectation bound and hence the proof.

\subsection{Proof of Remark \ref{rem:bound-rank}}
\label{subsec:proof-bound-rank}
    An argument akin to the proof of Theorem 7 in \cite{Bellec2017localizedgaussianwidthmconvex} shows that if we let $r=2\sqrt{\rank{\X}}$ and we denote the orthogonal projection onto the column space of $\X$ as $\Pi_\X$, we have $\EE\br{\norm{\Pi_\X\xi}_2}\leq \sqrt{\rank{\X}}$, and hence for every $\alphastar\in K$ we have
    \begin{equation*}
        w((\X K-\X\alphastar) \cap rB_2^d)=\EE\br{\sup_{\alpha\in (K-\alphastar),\norm{\X\alpha}_2\leq r}\inner{\xi}{\X\alpha}}\leq r\EE\br{\norm{\Pi_\X\xi}_2}\leq  r\sqrt{\rank{\X}}=\frac{r^2}{2}
    \end{equation*}
    implying by \eqref{eq:localized-inequality} that $\rzero^2(\X K)\leq 4\rank{\X}$.

\subsection{Proof of Corollary \ref{cor:comparison-LSE}}
\label{subsec:proof-comparison-LSE}

We begin by quantifying the constant $C$ in \cite[Corollary 1.2]{Chatterjee2014}: Temporarily denote $\rstar:=\rstar(\X\alphastar,\X \Kt)$.
By Theorem 1.1 in \cite{Chatterjee2014}, we know that
\begin{align*}
    \PP\pr{\abs{\norm{\X\pr{\haLSE-\alphastar}}_2-\rstar}\geq x\sqrt{\rstar}}&\leq 3\exp\pr{-\frac{x^4}{32(1+\frac{x}{\sqrt{\rstar}})^2}} \leq 3\exp\pr{-\frac{x^4}{32(1+x)^2}},
\end{align*}
where the second inequality is true if $\rstar\geq 1$.
From tail integration we get
\begin{align*}
    \EE_\xi\pr{\norm{\X(\haLSE-\alphastar)}_2-\rstar}^2 &= \int_0^\infty \PP\pr{\pr{\norm{\X(\haLSE-\alphastar)}_2-\rstar}^2\geq u}du \\
    &= 2\rstar \int_0^\infty x\PP\pr{\abs{\norm{\X\pr{\haLSE-\alphastar}}_2-\rstar}\geq x\sqrt{\rstar}}dx \\
    &\leq  6\rstar\int_0^\infty x \exp\pr{-\frac{x^4}{32(1+x)^2}}dx \\
    &\leq  125 \rstar
\end{align*}
and similarly, without the square,
\begin{align*}
    \EE_\xi\abs{\norm{\X(\haLSE-\alphastar)}_2-\rstar}
    &\leq  3\sqrt{\rstar}\int_0^\infty \exp\pr{-\frac{x^4}{32(1+x)^2}}dx \leq  18 \sqrt{\rstar}.
\end{align*}
Combining the two, we get
\begin{align*}
    \abs{\EE_\xi \norm{\X(\haLSE-\alphastar)}_2^2-\rstar^2}&= \abs{2\rstar \EE_\xi(\norm{\X(\haLSE-\alphastar)}_2-\rstar)+\EE_\xi(\norm{\X(\haLSE-\alphastar)}_2-\rstar)^2} \\
    &\leq 2\cdot 18\cdot \rstar^{3/2}+125\rstar \leq 161\rstar^{3/2}.
\end{align*}
Consequently, for every $\alphastar\in \Kt$ with $\rstar(\X\alphastar,\X\Kt)\geq (644/3)^2$, it holds $ \EE_\xi\cR(\haLSE)\geq \frac{1}{4}\frac{\rstar^2}{n}$, and in particular, if $\rstar(\X\Kt)\geq (644/3)^2$, then
\begin{equation}
    \label{eq:tight-bound-LSE}
     \sup_{\alphastar\in \Kt} \EE_\xi\cR(\haLSE) \geq \frac{\rstar^2(\X\Kt)}{4n}.
\end{equation}

Moreover, from \cref{thm:in-sample-risk-bound} we know that for $\rzero(\X\alphastar,\X K_{3c_a\tau})\geq 1$ and $\eps=\rzero(\X\alphastar,\X K_{3c_a\tau})/n$, in continuous time,
\begin{equation*}
    \EE_\xi \cR(\alphatstar) \leq \frac{7\rzero^2(\X\alphastar,\X K_{3c_a\tau})}{n}.
\end{equation*}
Hence, combining the two equations and using \cref{eq:critical-radius-scaling} yields
\begin{align}
    \sup_{\alphastar\in \Kt}\EE_\xi \cR(\alphatstar)&\leq \sup_{\alphastar\in \Kt} \frac{7\rzero^2(\X\alphastar, \X K_{3c_a\tau})}{n} \nonumber \\
    &\leq 21c_a \frac{ \rzero^2( \X K_{\tau})}{n} \tag{\text{from } \eqref{eq:critical-radius-scaling}} \\
    &\leq 21\cC c_a \cdot \frac{\rstar^2(\X K_{\tau})}{n} \tag{\text{by \cref{ass:unreal}}} \\
    &\leq 84 \cC c_a \cdot \sup_{\alphastar\in \Kt} \EE_\xi \cR (\haLSE), \tag{\text{from } \eqref{eq:tight-bound-LSE}}
\end{align}
which concludes the proof.

\subsection{Proof of Corollary \ref{cor:minimax-optimality}}
\label{subsec:proof-minimax-optimality}

The proof of \cref{cor:minimax-optimality} is analgous to the proof of \citet[Corollary 2.6]{Prasadan2024some}, where we can replace the upper bound on the LSE from their Proposition 2.4 with our bound on continuous-time ESMD from \cref{thm:in-sample-risk-bound} and use \cref{eq:critical-radius-scaling}.

Specifically, from \cref{thm:in-sample-risk-bound}, applying \cref{eq:critical-radius-scaling} and $c_a\lesssim 1$, we know that
\begin{equation*}
    \sup_{\alphastar\in \Kt}\EE_\xi \cR(\alphatstar)\lesssim  \frac{\rzero^2(\X K_{3c_a\tau})}{n} \lesssim \frac{\rzero^2(\X \Kt)}{n}.
\end{equation*}
Also note that in the proof of \cref{thm:in-sample-risk-bound} (\cref{subsec:proof-in-sample-risk-bound}), we showed that $\alphatstar\in 3c_a\Kt$, and so $\EE_\xi \cR(\alphatstar)\lesssim (\diam(\X \Kt))^2/n$.

Denote now temporarily
\begin{equation*}
    \Rbar:= \sup_{\alphastar\in \Kt} \EE_\xi \cR(\alphatstar).
\end{equation*}
Now, for every $\alphastar\in \Kt$, $r\mapsto \sup_{\alphastar\in \Kt} w(\X(\Kt-\alphastar)\cap rB_2^n)/r$ is non-increasing \citep[page 7]{Prasadan2024some}. Denoting $\rzero=\rzero(\X \Kt)$, we get that for some constant $c_1>0$, and some $r<\rzero$
\begin{align*}
    c_1\sqrt{n\Rbar}\leq r &\leq 2 \sup_{\alphastar\in \Kt} \frac{w(\X(\Kt-\alphastar)\cap r B_2^n)}{r} \tag{by \cref{def:critical-radius} and $r<\rzero$} \\
    &\leq 2  \frac{\sup_{\alphastar\in \Kt} w(\X(\Kt-\alphastar)\cap c_1\sqrt{n\Rbar} B_2^n)}{c_1\sqrt{n\Rbar}} \tag{non-increasing} \\
    &\leq 2 \max\mathset{1,\frac{1}{c_1}} \frac{\sup_{\alphastar\in \Kt} w(\X(\Kt-\alphastar)\cap \sqrt{n\Rbar} B_2^n)}{\sqrt{n\Rbar}}\\
    &\lesssim \sqrt{\log \locEntropy\pr{\sqrt{n\Rbar}, \X \Kt}} \tag{by \cref{ass:local-entropy}}
\end{align*}
Hence, $\sqrt{n\Rbar} \leq \sup\mathset{r>0: r\lesssim \sqrt{\log \locEntropy\pr{r, \X \Kt}}}$, which by the main result of \cite{Neykov2022minimax} yields minimax optimality of ESMD.

\subsection{Proof of Corollary \ref{cor:estimation-risk}}
\label{subsec:proof-estimation-risk}

    Consider the setting of \cref{thm:in-sample-risk-bound}, and assume that $\X$ has vanishing kernel width satisfying
    \begin{equation*}
        \frac{1}{n}\norm{\X\alpha}_2^2\geq \norm{\alpha}_2^2-f(3c_a\tau K,n)\quad \text{for all }\alpha\in 3c_a\tau\pr{K-K}.
    \end{equation*}
    Per assumption, we have that $f(K_{3c_a\tau },n)\lesssim \rzero^2(\X K_{3c_a\tau })/n$, which combined with the bound from \cref{thm:in-sample-risk-bound} and the fact that $\alphatstar\in 3c_a\tau K$ (see the proof of \cref{thm:in-sample-risk-bound}) yields for all $\alphastar\in \Kt$
    \begin{equation*}
        \norm{\alphastar-\alphatstar}_2^2\leq \frac{1}{n}\norm{\X(\alphatstar-\alphastar)}_2^2+f(3c_a\tau K,n)\lesssim \frac{\rzero^2}{n}+\frac{\log(1/\delta)}{n} + \frac{\rzero^2}{n}\asymp \frac{\rzero^2}{n}+\frac{\log(1/\delta)}{n},
    \end{equation*}
    where we denoted $\rzero = \rzero(\X K_{3c_a\tau })$.
    Applying \cref{eq:critical-radius-scaling} concludes the proof of the corollary.

\subsection{Proof of Proposition \ref{prop:lp-norm-squared}}
\label{subsec:proof-lp-norm-squared}

    Let $p\in(1,2)$ and $\varphi_K=\norm{\cdot}_p$. We first show that $\psi=\norm{\cdot}_p^2$ satisfies the discrete-time version of \cref{ass:sufficient-conditions-psi}.

    \paragraph{\ref{ass:suff-1}.} First, $\psi(\cdot)=\norm{\cdot}_p^2$ with $p\in(1,2)$ is differentiable \citep[Lemma 17][Example 3.2]{Shalev2007online,Juditsky2008large} with 
    \begin{equation*}
         \nabla \psi(\alpha)= 
         \begin{cases}
             2\frac{\sign{\alpha_i}\abs{\alpha_i}^{p-1}}{\norm{\alpha}_p^{p-2}} & \alpha \neq 0 \\
             0 & \alpha = 0
         \end{cases}       
         \quad \text{and} \quad 
         (\nabla \psi)^{-1}(\alpha)= 
         \begin{cases}
             \frac{\sign{\alpha_i}\abs{\alpha_i}^{q-1}}{2\norm{\alpha}_q^{q-2}} & \alpha \neq 0 \\
             0 & \alpha = 0
         \end{cases} 
         \quad \text{with} \quad \frac{1}{p}+\frac{1}{q}=1,
    \end{equation*}
    see for example \cite[Lemma 1]{Gentile2003}. Hence, $\nabla \psi(0)=0$.

    \paragraph{\ref{ass:suff-2}.} $\sqrt{\psi(\cdot)}=\norm{\cdot}_p$ is convex, as it is a norm.
    
    \paragraph{\ref{ass:suff-3}.} $\psi$ is $2(p-1)$-strongly convex with respect to the $\ell_p$-norm by \cite[Lemma 17]{Shalev2007online}.

    \paragraph{\ref{ass:suff-4}} This clearly holds with $c_u=c_l=1$.

    We bound $\rzero(\X K_{3\tau})=\rzero(3\tau\X B_p^d)$ under column-normalized design \eqref{eq:column-normalization}. We can use $\rzero^2(3\tau \X B_p^d)\leq 6\tau w(\X B_p^d)$ and then further bound $w(\X B_p^d)$. We can apply Hölder's inequality with $1/p+1/q=1$ and get, if $p> 1+1/\log d$
    \begin{equation*}
        w(\X B_p^d)=\EE\br{\sup_{\alpha\in B_p^d}\inner{\xi}{\X\alpha}} \leq \EE\br{\sup_{\alpha\in B_p^d}\norm{\X^\top\xi}_q\norm{\alpha}_p} \leq \EE\br{\norm{\X^\top\xi}_{q}} \leq \pr{\EE\br{\norm{\X^\top\xi}_{q}^q}}^{1/q}.
    \end{equation*}
    By column normalization, and since $\abs{\inner{\xi}{\X_j}} \stackrel{d}{=} \norm{\X_j}_2\abs{Z}$ for $Z\sim\cN(0,1)$, we get that
    \begin{equation*}
        \pr{\EE\br{\norm{\X^\top\xi}_{q}^q}}^{1/q} = (\EE \abs{Z}^q)^{1/q} \pr{\sum_{j=1}^d\norm{\X_j}_2^q}^{1/q} \asymp \sqrt{q} \sqrt{n} d^{1/q}.
    \end{equation*}
    On the other hand, if $p\leq 1+1/\log d$ and $d\geq 2$, we have that
    \begin{equation*}
        w(\X B_p^d)=\EE\br{\sup_{\alpha\in B_p^d}\inner{\xi}{\X\alpha}} \leq \EE\br{\sup_{\alpha\in B_p^d}\norm{\X^\top\xi}_q\norm{\alpha}_p} \leq d^{1/q}\EE\br{\norm{\X^\top\xi}_{\infty}} \leq d^{1/q}2\sqrt{n\log d} \leq 2e\sqrt{n\log d}.
    \end{equation*}
    where in the second last inequality we used that since $\X^\top\xi\sim \cN(0,\X^\top\X)$, and $\X$ is column-normalized, we know that $(\X^\top\xi)_i,i\in[d]$ is Gaussian with variance $n$, so that we may apply standard sub-Gaussian concentration, such as \cite[Exercise 2.5.10]{Vershynin2018} or \cite[Exercise 2.12]{Wainwright2019}, and in the last inequality we used that for $p\leq 1+1/\log d$ it holds that $d^{1/q}=d^{1-1/p}\leq d^{1/\log d}\leq e$.
    Hence, 
    \begin{equation}
        \label{eq:p-norm-radius-bound}
        \rzero^2(3\tau \X B_p^d)\lesssim 
    \begin{cases} 
         \tau \sqrt{n\log d} & \text{if } p\leq 1+1/\log d \\
         \tau \sqrt{q}\sqrt{n}d^{1/q} & \text{if } p> 1+1/\log d
    \end{cases}
    \end{equation}
    and from \cref{thm:in-sample-risk-bound} and \cref{rem:bound-rank} we get that
    \begin{align*}
        \cR(\alphatstar) &\leq \frac{4\rzero^2(\X K_{3\tau})}{n}+\frac{8\log(1/\delta)}{n} \\
        &\lesssim \frac{\rank{\X}}{n}\wedge 
        \frac{\tau}{\sqrt{n}}
        \begin{cases}
            \sqrt{\log d} & \text{if } 1<p\leq 1+\frac{1}{\log d},\\
            \sqrt{q} d^{1/q} & \text{if } 1+\frac{1}{\log d}< p < 2.
        \end{cases}
    \end{align*}

    To bound $\rzero(3\tau \X B_p^d)$ in under Gaussian design, we use \cite[Exercise 7.5.4]{Vershynin2018} and get
    \begin{equation*}
        \rzero^2(\X B_p^d)\leq 2 w(\X B_p^d)\leq 2\norm{\X}_2 w(B_p^d).
    \end{equation*}
    By \cite[Theorem 6.1]{Wainwright2019} we have that
    \begin{equation*}
        \PP\pr{\norm{\X}_2 \leq 3\sqrt{n} } \geq 1-\exp\pr{-n/2}
    \end{equation*}
    and for $1/q+1/p=1$ it is easily verified that, because for small enough $p$ we have $B_1^d\subset B_p^d \subset eB_1^d$ \citep{Lecue2017regularization,Vershynin2018},  the Gaussian width is bounded as
    \begin{equation*}
        w(B_p^d) \lesssim 
        \begin{cases}
            \sqrt{\log d} & \text{if } 1<p\leq 1+\frac{1}{\log d}, \\
            \sqrt{q}d^{1/q} & \text{if } 1+\frac{1}{\log d}< p <2,
        \end{cases}
    \end{equation*}
    so that with probability $1-\exp(-n/2)$ over draws of $\X$ the bound \cref{eq:p-norm-radius-bound} holds again. 
    This concludes the proof.

\subsection{Proof of Theorem \ref{thm:minimax-rate-1-2}}
\label{subsec:proof-minimax-rate-1-2}

Let $p\in(1,2)$ and $1/p+1/q=1$.  We introduce some further notation just for this proof.
We use $N(\eps,S)$, $M(\eps,S)$ to denote the $\eps$-covering number and, respectively, the $\eps$-packing number of a set $S$ with respect to the $\ell_2$-norm \citep[Definitions 5.1 and 5.4]{Wainwright2019}. $\log N(\eps,S)$ is commonly referred to as the metric entropy of $S$. We assume basic knowledge about covering and packing numbers that can be found, e.g., in Chapter 5 in \cite{Wainwright2019}.

\paragraph{Fixed design.} We will construct a hard design matrix $\X$ and reducing the problem to a Gaussian sequence model over a convex constraint set in $\RR^n$.

Let $m\in\NN$ and $k\in\NN$ be defined as
\begin{equation*}
    m = \floor{\frac{d^{1/p}}{\sqrt{n}} } \asymp\frac{d^{1/p}}{\sqrt{n}}\qquad \text{and} \qquad k = \floor{ \sqrt{n}d^{1/q}} \asymp \sqrt{n}d^{1/q}.
\end{equation*}
where we may write ``$\asymp$'' in the first case, because we assume that $ d^{1/p}\geq \sqrt{n}$.
We note that:
\begin{enumerate}
    \item By definition of $k$ and $m$, it holds that $k\cdot m \leq d^{1/p}d^{1/q}= d $.
    \item By definition of $k$, it holds that $k\leq n$ if $\sqrt{n}\geq d^{1/q}$, which we assumed. 
\end{enumerate}

Define $\ind_m = (1,\dots,1)^\top\in \RR^m$ as well as the all-zeros matrix $\zeros_{n\times d}\in\RR^{n\times d}$. We define the data matrix 
\begin{equation}
\label{eq:adversarial-matrix}
    \X=\sqrt{n}
    \begin{pmatrix}
        \A_{k\times km} & \zeros_{k\times(d-km)} \\
        \zeros_{(n-k)\times km} & \zeros_{(n-k)\times(d-km)}
    \end{pmatrix}
    \in\RR^{n\times d} \quad \text{with} \quad 
    \A_{k\times km}=
    \begin{pmatrix}
        \ind_m^\top & \zeros_{1\times m} & \hdots &\zeros_{1\times m}\\
        \zeros_{1\times m}  & \ind_m^\top & \hdots & \zeros_{1\times m}\\
        \vdots & \vdots & \ddots & \vdots \\
        \zeros_{1\times m} & \zeros_{1\times m} & \dots & \ind_m^\top
    \end{pmatrix}
    \in \RR^{k\times km}.
\end{equation}
Note that the columns of $\X$ are normalised to $\norm{\X_i}_2\leq \sqrt{n}$.

As a first step, we reduce our problem to a more well-studied Gaussian sequence model.
Define $t:=\sqrt{n}m^{1/q}$ and $x_i, v_i\in \RR^d,i\in[k]$ as 
\begin{equation*}
    x_i=\sqrt{n}(\zeros_{1\times m(i-1)},\ind_m^\top,\zeros_{1\times (d-mi)})^\top \quad \text{and} \quad v_i=\frac{m^{-1/p}}{\sqrt{n}} x_i,
\end{equation*}
where $x_i$ is the vector corresponding to the $i$-th row of $\X$. Note that $\norm{v_i}_p=\norm{ \frac{m^{-1/p}}{\sqrt{n}} x_i }_p= m^{-1/p} m^{1/p}=1$ and thus $v_i\in  B_p^d$.
Denoting the $i$-th standard-basis vector in $\RR^n$ as $\e_i$, we have that
\begin{equation*}
    \X v_i = \frac{m^{-1/p}}{\sqrt{n}} \X x_i =  \sqrt{n} m^{1-1/p}  \e_i = \sqrt{n} m^{1/q}\e_i = t\e_i
\end{equation*}
for all $i\in[k]$. Taking the $p$-convex hull of the $v_i$ amounts to choosing any $\delta_i,i\in[k]$ such that $\sum_{i=1}^k\abs{\delta_i}^p\leq1$ and letting $v_\delta = \sum_{i=1}^k\delta_i v_i$. It follows that $v_\delta\in B_p^d$ and
\begin{equation*}
    \X v_\delta = \sum_{i=1}^k\delta_i \X v_i = t \sum_{i=1}^k\delta_i \e_i.
\end{equation*}
Because $B_p^k\times \mathset{0}^{n-k}$ can be decomposed into $\sum_{i=1}^k\delta_i\e_i$, we have that
\begin{equation*}    
  S_t:=tB^k_p\times \mathset{0}^{n-k} \subset \X B^d_p\subset \RR^n.
\end{equation*}

Therefore, since $y=\X\alphastar+\xi$ with $\xi\sim \cN(0,\I_n)$ and $S_t\subset \X B_p^d$, the minimax rate for estimation in the Gaussian sequence model $y=\theta^\star+\xi$, with $\theta^\star\in S_t$ is a lower bound for our original problem, that is,
\begin{equation}
\label{eq:reduction-Gaussian-sequence}
    \inf_{\ha}\sup_{\alphastar\in  B_p^d}\EE\br{\norm{\X(\alphastar-\ha)}_2^2}\geq \inf_{\widehat \theta}\sup_{\theta^\star\in S_t}\EE\br{\|\theta^\star-\widehat \theta\|_2^2}.
\end{equation}

We may now use the following result from \citet[Theorem 11.7]{johnstone2017gaussian}, see also \citet[page 6]{aolaritei2025revisiting}:
Let $p\in(1,2)$ and consider the Gaussian sequence model $y = \theta^\star+\xi$ where $\xi\sim \cN(0,\sigma^2 \I_d)$. Then, if $\sigma^2 \in [\frac{1}{d^{2/p}}, \frac{1}{1+\log d}]$, the minimax rate over $B_p^d$ is given by
\begin{equation*}
    \mathfrak{M}(B_p^d,\sigma) := \inf_{\widehat\theta} \sup_{\theta^\star\in B_p^d} \EE_\xi\br{\|\widehat\theta(y)-\theta^\star\|_2^2} \asymp (\sigma^2 \log(e d \sigma^p))^{1-p/2}.
\end{equation*}

To apply this result to \cref{eq:reduction-Gaussian-sequence}, we use the following identity
\begin{equation*}
    \mathfrak{M}(\tau B_p^d,\sigma) = \tau^2 \cdot \mathfrak{M}\pr{ B_p^d,\frac{\sigma}{\tau}},
\end{equation*}
which follows from simple rescaling arguments. Moreover, we rely on the fact that if a convex body 
$K$ (here $S_t$) is contained in a linear subspace, then the minimax risk of the Gaussian sequence model over $K\subset\RR^d$ coincides with that of the model restricted to the subspace, since the orthogonal projection onto the subspace constitutes a sufficient statistic for the parameter. In particular, we can bound \cref{eq:reduction-Gaussian-sequence} as
\begin{equation*}
    \inf_{\widehat \theta}\sup_{\theta^\star\in S_t}\EE\br{\|\theta^\star-\widehat \theta\|_2^2} \geq \mathfrak{M}(t B_p^k,1) = t^2 \cdot \mathfrak{M}(B_p^k,1/t) = t^{p} \pr{\log\pr{ek/t^p}}^{1-p/2}
\end{equation*}
where the last equality holds because
\begin{equation*}
    t^p=\pr{\sqrt{n}m^{1/q}}^p \asymp_p \pr{\sqrt{n}\pr{\frac{d^{1/p}}{\sqrt{n}}}^{1/q}}^p = (\sqrt{n})^{p-p/q} d^{1/q}=\sqrt{n}d^{1/q} = k \quad \implies \quad \pr{\log\pr{ek/t^p}}^{1-p/2} \asymp_p 1,
\end{equation*}
and in particular, $1/t^2 \asymp_p 1/k^{2/p}  \in [\frac{1}{k^{2/p}}, \frac{1}{1+\log k}]$. Plugging this into the lower bound and dividing by $n$ proves that the minimax rate on this data matrix is lower bounded by $c_p d^{1/q}/\sqrt{n}$, which concludes the first part.

    \paragraph{Gaussian design.}
    Let $p\in[1+c_1/\log \log d,2)$, $1/p+1/q = 1$ and $ n(\log d)^{c_0} \leq d \leq n^{q/2}$. Suppose $\X\in\RR^{n\times d}$ is a standard Gaussian matrix.

    We begin by noting that for some constant $c>0$, the event 
    \begin{equation*}
        \cE_1 = \mathset{\sqrt{d}-c\sqrt{n} \leq \sigma_{\min}(\X)\leq \sigma_{\max}(\X) \leq \sqrt{d}+c\sqrt{n}}
    \end{equation*}
    has probability $\PP\pr{\cE_1}\geq 1-2\exp\pr{-n}$ \citep[Theorem 4.6.1]{Vershynin2018}.
    Since $d\geq n(\log d)^{c_0}$ we have that $\sigma_{\min}(\X)\asymp \sqrt{d}$ with probability at least $1-2\exp\pr{-n}$. Since $\norm{\cdot}_p \leq d^{1/p-1/2} \norm{\cdot}_2$, we know that $d^{1/2-1/p} B_2^d\subset B_p^d$, and hence on $\cE_1$ that for some constant $c>0$
    \begin{align*}
          c \sqrt{d} d^{1/2-1/p} B_2^n = c d^{1/q} B_2^n \subset \X B_p^d.
    \end{align*}
    
    Following \cite{Neykov2022minimax}, conditioned on $\X$, we need to solve the stationary condition of 
    \begin{equation*}
        \log M(\eps,\mathbf{X}B_p^d  \cap C\eps B_2^n) \asymp \log M(\eps,\mathbf{X}B_p^d )  \asymp \eps^2.
    \end{equation*}
    to determine the minimax rate given $\X$, where the first ``$\asymp$'' holds by a pigeon-hole argument.
    We solve this by first noting that on the event $\cE_1$, the solution must satisfy $\eps \gtrsim d^{1/q}$, since on $\cE_1$ and if $\eps \asymp d^{1/q}$ we have
    \begin{equation*}
        \log M (\eps , \X B_p^d ) \geq \log M (\eps , c\eps B_2^n ) \gtrsim n \geq d^{2/q} \asymp \eps^2.
    \end{equation*}
    Thus, all we need is to estimate $ \log M(\eps,\mathbf{X}B_p^d ) $ when $\eps \gtrsim d^{1/q}$. This was done (implicitly) in \cite{Kur2024minimum}, when $p \geq 1+ c_1/\log\log(d)$, $ d \gtrsim n(\log d)^{c_0}$ and $n$ is sufficiently large. Specifically, it was shown that if $ \eps \gtrsim  d^{1/q}$, the event
    \begin{equation}\label{eq:coveringlp}
                 \cE_2 = \mathset{\log M(\eps,\mathbf{X}B_p^d  ) \asymp \log M(\eps/\sqrt{n}, B_p^d )}
    \end{equation}
    has probability at least $1-C\exp\pr{-cn}$. 
    We outline the proof of \eqref{eq:coveringlp} after finishing the main proof.
    Using \eqref{eq:coveringlp}, for $\eps \gtrsim d^{1/q}$ we can solve the inequality
    \begin{equation*}
        \log M(\eps /\sqrt{n}, B_p^d )  \asymp \pr{\frac{\eps}{\sqrt{n}}}^{-\frac{2p}{2-p}}\log \pr{d\pr{\frac{\eps}{\sqrt{n}}}^{\frac{2p}{2-p}}} \gtrsim \eps^2,
    \end{equation*}
    where we used \eqref{eq:metric-entropy-lp} from \cite{Schutt1984entropy}:
    \begin{equation}
    \label{eq:metric-entropy-lp}
        \log N(\eps,B_p^k)\asymp_p
        \begin{cases}
            \eps^{-\frac{2p}{2-p}}\log\pr{k \eps^{\frac{2p}{2-p}}} & \text{if } \eps \gtrsim k^{1/2-1/p}, \\
            k\log \pr{k^{-1}\eps^{-\frac{2p}{2-p}}} & \text{if } \eps \lesssim k^{1/2-1/p},
        \end{cases}
    \end{equation}
    We can see that choosing $\eps^2 = n^{p/2}(\log d)^{1-p/2}$ satisfies this inequality.
    It follows from dividing by $n$ and taking the union bound over $\cE_1,\cE_2$ that with probability $1-c_2\exp(-c_3n)$ over draws of the data matrix $\X$, we have
    \begin{equation*}
        \inf_{\widehat\alpha}\sup_{\alphastar \in B_p^d} \EE_\xi \cR(\widehat\alpha) \asymp n^{p/2-1} (\log d)^{1-p/2}  \vee \frac{d^{2/q}}{n},
    \end{equation*}
    which finishes the proof.
    
\begin{proof}[Proof of \eqref{eq:coveringlp}]
    We outline the proof of \eqref{eq:coveringlp} based on arguments from \cite{Kur2024minimum} here again for completeness.     
    Define the $i$-th dyadic entropy number of any set $S$ as 
    \begin{equation*}
        e_i(S)=\inf\mathset{\eps>0 : \log_2 N(\eps,S) \leq i-1}.
    \end{equation*}
    By definition, we have that $e_i(S)\leq \eps$, if and only if $\log N(\eps,S)\leq i$. We know by Schütt's Theorem \citep{Schutt1984entropy} that the $i$-th dyadic entropy number of $B_p^k$ is given by
    \begin{equation}
    \label{eq:dyadic-entropy-lp}
        e_i(B_p^k) \asymp 
        \begin{cases}
            1 & \text{if } 1\leq i \leq \log k,\\
            \pr{\frac{\log(ek/i)}{i}}^{1/p-1/2} & \text{if } \log k \leq i \leq k, \\
            2^{-i/k}k^{1/2-1/p} & \text{if } i\geq k.
        \end{cases}
    \end{equation}
    
    For $i \in [d]$ with $\log d \leq i \leq d$, define $\Sigma_i = \mathset{\sigma\subset [d]\setmid \abs{\sigma}\asymp i/\log(ed/i)}$
    and $\eps_i := e_i(B_p^d) \asymp (\log(ed/i)/i)^{1/p-1/2}$ (recall \eqref{eq:dyadic-entropy-lp}), and define the sets
    \begin{equation*}
        \cV_i = \bigcup_{\sigma \in \Sigma_i} A_\sigma \quad \text{where} \quad A_\sigma =   \frac{\eps_i}{\sqrt{\abs{\sigma}}} \cdot B_\infty^\sigma = \frac{\eps_i}{\sqrt{\abs{\sigma}}} \mathset{v\in \RR^d\setmid v_i= 0 \text{ if }i\notin \sigma, \norm{v}_\infty \leq 1}.
    \end{equation*}
    Then it is easy to see that $\cV_i \subset B_p^d$, as for all $v\in \cV_i$ we have
    \begin{equation*}
        \norm{v}_p \leq \abs{\sigma}^{1/p} \norm{v}_\infty \leq \abs{\sigma}^{1/p-1/2} \eps_i \lesssim \pr{\frac{i}{\log (ed/i)}}^{1/p-1/2} \pr{\frac{\log(ed/i)}{i}}^{1/p-1/2} \leq 1.
    \end{equation*}
    Also note that by the same calculation, for $v\in \cV_i$, it holds $\norm{v_i}_2\leq \eps_i$.

    Let $\cN_i$ be an $\eps_i$-covering of $\cV_i$ and note that \citep[underneath Equation (25)]{Kur2024minimum} we get $\log \abs{\cN_i} \lesssim i$, meaning it has the same metric entropy for $\eps_i$ as $B_p^d$.
    Following \cite{Schutt1984entropy,Kur2024minimum}, one can see that for $\cI := \{2^0,2^1,2^2,\ldots,c \cdot d \}$, every  $v \in B_{p}^d$ can be written as
    \begin{equation*}  
        v = v_0 + \sum_{i\in \cI}\delta_i v_i  \qquad \text{where} \qquad \norm{v_0}_2 \leq \eps_{d} \lesssim d^{1/q - 1/2} 
    \end{equation*}
    and $v_i\in \cN_i$, $\|v_i\|_p\asymp 1 $, $\sum_{i\in \cI}\delta_i^{p} \leq 1$. Consider the partition of $\cI=\cI_1\cup \cI_2$ with $\cI_1 = \{2^0,2^1,2^2,\ldots,c n \}$, $\cI_2 =  \{n,2^1n,2^2n,\ldots,c \cdot d\}$ (where w.l.o.g., we assume these terms are powers of two).
    By triangle inequality, it follows that we can bound the norm of $\X v$ for any $v\in B_p^d$ as 
    \begin{equation*}
        \norm{\X v}_2 \leq \norm{\X v_0}_2 + \norm{\sum_{i\in \cI_1} \delta_i \X v_i}_2 + \norm{\sum_{i\in \cI_2} \delta_i \X v_i}_2.
    \end{equation*}
    We now bound each term separately.

    To that end, recall that by the Johnson-Lindenstrauss Lemma \citep[Theorem 5.3.1 and Exercise 5.3.3]{Vershynin2018}, we know that for $\delta \in (0,1)$ and any finite set $\cN\subset \RR^d$ with $n\geq (C/\delta^2)\log \abs{\cN}$, it holds that
    \begin{equation}\label{eq:JL}
       \forall v,w\in \cN: \quad (1 - \delta) \cdot \|v-w\|_{2} \leq  \frac{1}{\sqrt{n}}\|\mathbf{X}(v-w)\|_{2} \leq (1+ \delta) \cdot \|v-w\|_{2}
    \end{equation}
    with probability of at least $1-2\exp(-cn\delta^2)$.

    \begin{enumerate}
        \item Note that as the maximal singular value satisfies $\sigma_{\max}(\mathbf X) \lesssim \sqrt{d}$ on the event $\cE_1$, we get that
        \begin{equation*}
            \norm{\mathbf{X} v_0}_2 \lesssim \sqrt{d} \cdot \varepsilon_{d} \lesssim  d^{1/q}.
        \end{equation*}
    
        \item By \eqref{eq:JL}, since $\log \abs{\bigcup_{i \in \cI_1}\cN_i} \lesssim n $, we obtain for $ \cN = \bigcup_{i\in \cI_1}\cN_i$ that
        \begin{equation*}
            \cE_3 = \mathset{\forall v,w\in \cN:\  \norm{\X (v-w)}_{2} \asymp \sqrt{n} \cdot \|v-w\|_{2}},
        \end{equation*}
        holds with probability at least $1-2\exp\pr{-cn}$. Clearly, on $\cE_3$, we also have that $\norm{\X v}_{2} \asymp \sqrt{n} \cdot \|v\|_{2}$ for all $v\in \cN$.
    
        \item Again, by \eqref{eq:JL} with $\delta\asymp \sqrt{i/n}$ so that $\log \abs{\cN_i} \lesssim i \asymp \delta^2 n$, the event
        \begin{equation*}
            \cE_4=\mathset{\forall i \in \cI_2:\ \sup_{v \in \cN_i}\norm{ \X v }_{2} \lesssim \sqrt{n} \cdot \sqrt{i/n}   \cdot \norm{v}_2 \leq  \sqrt{n} \cdot \sqrt{i/n}   \cdot \eps_i = \sqrt{i} \cdot \eps_i}.
        \end{equation*}
        holds with probability at least $1-2\sum_{i\in \cI_2}\exp(-c i)\geq 1-C\exp(cn)$. 
        Hence, by triangle and Hölder's inequalities, we obtain that 
        \begin{align*}
            \norm{\sum_{i \in \cI_2} \delta_i \mathbf{X} v_i }_{2} 
            \leq \sum_{i \in \cI_2} \sqrt{i} \, \varepsilon_i \cdot \delta_i \lesssim \left(\sum_{i \in \cI_2} \delta_i^p \right)^{1/p} \cdot \left(\sum_{i \in \cI_2} (\sqrt{i} \, \varepsilon_i)^{q} \right)^{1/q} \lesssim \left(\sum_{i \in \cI_2} (\sqrt{i} \, \varepsilon_i)^{q} \right)^{1/q}.
        \end{align*}
        Plugging in $\eps_i\lesssim \log(d/i)^{1/p-1/2} i^{1/2-1/p}$ from \eqref{eq:dyadic-entropy-lp} we get that this is bounded by
        \begin{equation*}
             \left(\sum_{i \in \cI_2} (\sqrt{i} \, \varepsilon_i)^{q} \right)^{1/q}\lesssim \left(\sum_{i \in \cI_2} \left( \log(d/i)^{1/p - 1/2} i^{1/q} \right)^{q} \right)^{1/q} \lesssim \log(d)^{2/q} \cdot d^{1/q} \lesssim d^{1/q}
        \end{equation*}
        where in the last inequality we used that $p \geq 1 + C/\log\log(d)$. 
    \end{enumerate}
    Note that if $\eps \gtrsim d^{1/q}$, we have that $\log M(\eps, \X B_p^d)\asymp \log M(\eps, \X B_p^d \setminus d^{1/q }B_2^n)$. 
    We just showed that on the intersection of events $\cE_1,\cE_3, \cE_4$ we have that if $\norm{\X v}_2 \geq d^{1/q}$, it holds that $d^{1/q} \leq \norm{\X v}_2 \lesssim \sqrt{n} \norm{\sum_{i\in \cI_1} \delta_i v_i}_2 + d^{1/q}$, and hence $\norm{\sum_{i\in \cI_1} \delta_i v_i}_2\gtrsim d^{1/q}/\sqrt{n}$.
    Therefore, the metric entropy beyond the threshold $d^{1/q}$ only depends on the $i$-th dyadic numbers of $B_p^d$ with $i\in \cI_1$. On this set, we already showed that the Johnson-Lindenstrauss Lemma applies (i.e., $\cE_3$), and hence
    \begin{equation*}
        \log M(\eps, \X B_p^d) \asymp \log M(\eps /\sqrt{n}, B_p^d)
    \end{equation*}
    with probability at least $1-C\exp\pr{-c n}$, which is \eqref{eq:coveringlp}.
\end{proof}

\subsection{Proof of Theorem \ref{thm:l1-all-combined}} 
\label{subsec:proof-l1-all-combined}

    We split the proof of \cref{thm:l1-all-combined} into several auxiliary lemmas for each potential, which we prove in \cref{subsec:proofs-auxiliary-l1-lemmas}.
    Specifically, we show for each potential from \cref{tab:l1-potentials} that it satisfies \cref{ass:sufficient-conditions-psi} with the parameters stated in \cref{tab:l1-potentials}. \cref{lem:lp-for-l1} provides the details for the squared $\ell_p$-norm, \cref{lem:l1-moreau} for the Moreau envelope-based potential, \cref{lem:squared-hypentropy} for the adjusted hypentropy and \cref{lem:sigmoidal-approximation} for the sigmoidal potential. We prove \cref{lem:lp-for-l1,lem:l1-moreau,lem:squared-hypentropy,lem:sigmoidal-approximation} in \cref{subsec:proofs-auxiliary-l1-lemmas}.

    \begin{lemma}
\label{lem:lp-for-l1}
    Suppose that $\varphi_K=\norm{\cdot}_1$ and $\psi=\norm{\cdot}_p^2$ with $1< p\leq 1+1/\log(d)$. Then the discrete time versions of \cref{ass:sufficient-conditions-psi} are satisfied with constants $c_l=e,c_u=1$ and $\rho=2(p-1)$ w.r.t.\ the $\ell_p$-norm.
\end{lemma}
\begin{lemma}
\label{lem:l1-moreau}
    Suppose that $\varphi_K=\norm{\cdot}_1$ and 
    \begin{equation*}
            \psi(\alpha) = \pr{\sum_{i=1}^d h_\lambda(\alpha_i)+\frac{d\lambda}{2}}^2+\frac{\rho}{2}\norm{\alpha}_2^2 \quad \text{with} \quad h_\lambda (x)= \begin{cases} \frac{x^2}{2\lambda} & \abs{x}\leq \lambda \\ \abs{x}-\frac{\lambda}{2} & \abs{x}> \lambda \end{cases}
    \end{equation*}
    with $\lambda \leq 2\tau /d $ and $\rho =2$.
    Then the discrete-time version of \cref{ass:sufficient-conditions-psi} is satisfied with $\rho=2$ with respect to the $\ell_2$-norm and constants $c_l=1, c_u = \sqrt{5}$.
\end{lemma}
\begin{lemma}
\label{lem:squared-hypentropy}
    Suppose that $\varphi_K=\norm{\cdot}_1$ and 
    \begin{align*}
        \psi(\alpha)= \pr{\sum_{i=1}^d \frac{\pr{ \alpha_i\arcsinh\pr{\alpha_i/\gamma}-\sqrt{\alpha_i^2+\gamma^2}+\gamma+1}}{\arcsinh\pr{\gamma^{-1}}}}^2,
    \end{align*} 
    with $\gamma\leq  \sinh(d/\tau)^{-1}\wedge (4\tau)^{-1}\wedge  2^{-1/2}$.
    Then it satisfies both the continuous and discrete-time versions of \cref{ass:sufficient-conditions-psi} with constants
    $c_l=1, c_u=3$ and $\rho = \arcsinh(\gamma^{-1})^{-2}$ with respect to both the $\ell_1$ and $\ell_2$-norm.
\end{lemma}
\begin{lemma}
\label{lem:sigmoidal-approximation}
Suppose that $\varphi_K=\norm{\cdot}_1$, and
\begin{equation*}
    \psi(\alpha)= \pr{\sum_{i=1}^d \frac{1}{\gamma}\pr{\log\pr{1+\exp\pr{-\gamma \alpha_i}}+\log\pr{1+\exp\pr{\gamma \alpha_i}}}}^2
\end{equation*}
with $\gamma \geq d\log (4)/\tau$.
Then $\psi$ satisfies the continous-time version of \cref{ass:sufficient-conditions-psi} with constants $c_l=1, c_u =2$.
\end{lemma}

The proofs of \cref{lem:lp-for-l1,lem:l1-moreau,lem:squared-hypentropy,lem:sigmoidal-approximation} can be found in \cref{subsec:proofs-auxiliary-l1-lemmas}.

    Therefore, we may invoke \cref{thm:in-sample-risk-bound}, and the result follows by directly bounding the stationary radius:
    Suppose that $K=B_1^d$. If $\X$ is column-normalized \eqref{eq:column-normalization} and $d\geq \tau \sqrt{n}$, then
    \begin{equation*}
        \frac{\rzero^2(\X K_{3c_a\tau })}{n} \leq 4\frac{\rank{\X}}{n} \wedge 186c_a\tau\sqrt{\frac{\log (2ed/(\tau\sqrt{n}))}{n}}.
    \end{equation*}
    This follows directly from the derivations in \cite[Theorem 7]{Bellec2017localizedgaussianwidthmconvex} or from \cite{Gordon2007}.  
    Plugging this into \eqref{eq:in-sample-risk-bound} from \cref{thm:in-sample-risk-bound} we get that with probability at least $1-\exp(-0.1n)-\delta$ it holds
    \begin{align*}
        \cR(\alphatstar) &\leq \frac{4\rzero^2(\X K_{3c_a\tau})}{n}+\frac{8\log(1/\delta)}{n} \\
        &\leq  4\pr{4\frac{\rank{\X}}{n} \wedge 186c_a\tau\sqrt{\frac{\log (2ed/(\tau\sqrt{n}))}{n}}} + \frac{8\log(1/\delta)}{n} \\
        &\leq  \pr{16\frac{\rank{\X}}{n} \wedge 744c_a\tau\sqrt{\frac{\log (2ed/(\tau\sqrt{n}))}{n}}} + \frac{8\log(1/\delta)}{n}
    \end{align*}

    To bound $\rzero(3c_a\tau \X B_1^d)$ in under Gaussian design,
    we use \cite[Examples 7.5.4 and 7.5.11]{Vershynin2018}
    which yield that 
    \begin{equation*}
        \rzero^2(\X B_1^d)\leq 2 w(\X B_1^d)\leq 2\norm{\X}_2 w(B_1^d) \lesssim \norm{\X}_2 \sqrt{\log d}.
    \end{equation*}
    By \cite[Theorem 6.1]{Wainwright2019} we have that
    \begin{equation*}
        \PP\pr{\norm{\X}_2 \leq 3\sqrt{n} } \geq 1-\exp\pr{-n/2}
    \end{equation*}
    so that with probability $1-\exp(-n/2)$ it holds $\rzero^2(3c_a\tau \X B_1^d)\lesssim c_a\tau \sqrt{n\log d}$.
    Combining this with \cref{thm:in-sample-risk-bound} and \cref{rem:bound-rank}, and noting that $\rank{\X}=n$ with probability $1$, we get from a union bound that with probability $1-2\exp\pr{-0.1n}-\delta$ that
    \begin{align*}
        \cR(\alphatstar) &\leq \frac{4\rzero^2( 3c_a\tau\X B_1^d)}{n}+\frac{8\log(1/\delta)}{n} \\
        &\lesssim \pr{1\wedge c_a\tau \sqrt{\frac{\log d}{n}}}+\frac{\log(1/\delta)}{n}
    \end{align*}
    Choosing $\delta = 0.01$ so that $1-2\exp\pr{-0.1n}-\delta = 0.99-2\exp\pr{-0.1n}$, we get that
    \begin{equation*}
        \cR(\alphatstar)\lesssim 1\wedge c_a\tau\sqrt{\frac{\log d}{n}}.
    \end{equation*}
    This concludes the proof.

\begin{remark}
    As a side-remark, we also provide a direct proof that if $\X$ is column-normalized \eqref{eq:column-normalization}, then
    \begin{equation*}
        \frac{\rzero^2(\X K_{3c_a\tau})}{n} \leq 4\frac{\rank{\X}}{n} \wedge 12 c_a\tau \sqrt{\frac{\log d}{n}}.
    \end{equation*}
    Note that the bound in terms of the rank of $\X$ follows from \cref{rem:bound-rank} and \cref{subsec:proof-bound-rank}.
    We can use $\rzero^2(\X \Kt)\leq 2w(\X \Kt)$ and to further bound $w(\X \Kt)$, we can apply Hölder's inequality and get for $d\geq 2$
    \begin{equation*}
        w(\X \Kt)=\EE\br{\sup_{\alpha\in \Kt}\inner{\xi}{\X\alpha}} \leq \EE\br{\norm{\X^\top\xi}_{\infty}}\sup_{\alpha\in \Kt}\norm{\alpha}_1 = \tau\EE\br{\norm{\X^\top\xi}_{\infty}}\leq 2\tau \sqrt{n\log d}, 
    \end{equation*}
    where in the last step we used that $\X^\top\xi\sim \cN(0,\X^\top\X)$, and $\X$ is column-normalized \eqref{eq:column-normalization}, which implies that $(\X^\top\xi)_i,i\in[d]$ is Gaussian with variance $n$, so that we may use standard sub-Gaussian concentration \citep[Exercise 2.12]{Wainwright2019}, which yields the second upper bound.
    Plugging the second bound into \eqref{eq:in-sample-risk-bound} from \cref{thm:in-sample-risk-bound} we get that with probability at least $1-\exp(-0.1n)-\delta$ it holds
    \begin{align*}
        \cR(\alphatstar) &\leq \frac{4\rzero^2(\X K_{3c_a\tau})}{n}+\frac{8\log(1/\delta)}{n} \\
        &\leq  4\pr{4\frac{\rank{\X}}{n} \wedge 12c_a\tau\sqrt{\frac{\log d}{n}}} + \frac{8\log(1/\delta)}{n} \\
        &\leq  \pr{16\frac{\rank{\X}}{n} \wedge 48c_a\tau\sqrt{\frac{\log d}{n}}} + \frac{8\log(1/\delta)}{n}.
    \end{align*}
\end{remark}

\subsection{Proof of Proposition \ref{prop:M-convex-hulls}}
\label{subsec:proof-M-convex-hulls}

    We check that \cref{ass:sufficient-conditions-psi} is satisfied. We use the results from \cite[Example 4.5]{Beck2012smoothing}.
    
    \paragraph{\ref{ass:suff-1}.} The differentiability of $\psi$ is clear, with gradient and Hessian of $\psi$ given by
    \begin{align*}
        \nabla\psi(x) &= \frac{2}{\gamma} \log\pr{ m\sum_{i=1}^m\exp(\gamma\inner{a_i}{x})} \mu(x) + \rho x \\
        \nabla^2\psi(x) &= 2\log\pr{m\sum_{i=1}^m\exp(\gamma\inner{a_i}{x})}\pr{\sum_{i=1}^m p_i(x)a_ia_i^\top-\mu(x)\mu(x)^\top} +2\mu(x)\mu(x)^\top+ \rho\I_d
    \end{align*}
    where we defined the quantities
    \begin{equation*}
        p_i(x) = \frac{\exp(\gamma\inner{a_i}{x})}{\sum_{j=1}^m \exp(\gamma\inner{a_j}{x})} \quad \text{and} \quad \mu(x) = \sum_{i=1}^m p_i(x) a_i
    \end{equation*}
    Therefore, since $p_i(0)=1/m$ we have $\mu(0)=\frac{1}{m}\sum_{i=1}^m a_i $ which we assumed to be zero, and hence $\nabla \psi(0)=0$.

    \paragraph{\ref{ass:suff-2}.} The convexity of $\sqrt{\psi}$ follows as previously because $\sqrt{\psi}$ is the Euclidean norm of $\sqrt{\frac{\rho}{2}}\norm{x}_2$ and $\frac{1}{\gamma}\log \pr{m\sum_{i=1}^m \exp\pr{\gamma \inner{a_i}{x}}}$, of which the Hessian 
    is given by
    \begin{equation*}
        \pr{\sum_{i=1}^m p_i(x)a_ia_i^\top-\mu(x)\mu(x)^\top}
    \end{equation*}
    which is clearly positive semi-definite for every $x\in\RR^d$ (as it is the covariance matrix of the discrete distribution over $a_i$ with probabilities $p_i(x)$) and thus convex.

    \paragraph{\ref{ass:suff-3}.} The $\rho$-strong convexity follows as previously because
    \begin{equation*}
        x\mapsto \psi(x) - \frac{\rho}{2}\norm{x}_2^2 = \frac{1}{\gamma^2}\log^2 \pr{m\sum_{i=1}^m \exp\pr{\gamma \inner{a_i}{x}}}
    \end{equation*}
    is convex (which can be seen by the Hessian above).

    \paragraph{\ref{ass:suff-4}.} By \cite[Example 4.5]{Beck2012smoothing}, we know that for all $x\in\RR^d$
    \begin{equation*}
        \varphi_K(x)-\frac{1}{\gamma}\log m \leq \frac{1}{\gamma}\log \pr{\sum_{i=1}^m \exp\pr{\gamma \inner{a_i}{x}}} \leq \varphi_K(x) +\frac{1}{\gamma}\log m.
    \end{equation*}
    It follows that
    \begin{equation*}
        \varphi_K(x) \leq \frac{1}{\gamma}\log \pr{m\sum_{i=1}^m \exp\pr{\gamma \inner{a_i}{x}}} \leq \varphi_K(x) +\frac{2}{\gamma}\log m.
    \end{equation*}
    and therefore, for all $\alpha\in \RR^d$ and $\alpha'\in \tau K$ with $\rho = 2/ \max_{i\in[M]} \norm{k_i}_2^2$ and $\gamma \geq 2\log (m)/\tau$
    \begin{align*}
        \varphi_K(\alpha) & \leq  \frac{1}{\gamma}\log \pr{m\sum_{i=1}^m \exp\pr{\gamma \inner{a_i}{\alpha}}} \leq c_l\sqrt{\psi(\alpha)} \\
        \sqrt{\psi(\alpha')} &= \sqrt{\frac{1}{\gamma^2}\log^2 \pr{m\sum_{i=1}^m \exp\pr{\gamma \inner{a_i}{\alpha'}}} +\frac{\rho}{2}\norm{\alpha'}_2^2} \leq \varphi_K(\alpha')+\frac{2}{\gamma}\log m +\tau \leq 3\tau = c_u\tau
    \end{align*}
    with $c_l=1$, $c_u = 3$.

    Using \cite[Proposition 5]{Bellec2017localizedgaussianwidthmconvex} we get that if $M\geq \tau \phi \sqrt{n}$, then
    \begin{equation*}
        \rzero^2(\X \Kt) \leq 31 \tau \phi \sqrt{n}\sqrt{\log\pr{\frac{eM}{\tau\phi \sqrt{n}}}}.
    \end{equation*}
    Hence, from \cref{thm:in-sample-risk-bound} and \cref{rem:bound-rank}, we know that with probability $1-\exp(-0.1n)-\delta$ it holds that
    \begin{align*}
        \cR(\alphatstar) &\leq \frac{4\rzero^2(\X K_{3c_a\tau})}{n}+\frac{8\log(1/\delta)}{n} \\
        &\leq 4\pr{\frac{4\rank{\X}}{n}\wedge 93 c_a \tau \phi \sqrt{\frac{\log\pr{eM/(\tau\phi \sqrt{n})}}{n}}} +\frac{8\log(1/\delta)}{n} \\
        &=\pr{\frac{16\rank{\X}}{n}\wedge 372 c_a \tau \phi \sqrt{\frac{\log\pr{eM/(\tau\phi \sqrt{n})}}{n}}} +\frac{8\log(1/\delta)}{n}.
    \end{align*}
    Choosing $\delta=0.01$ yields the result.
    This concludes the proof.

\subsection{Proofs of the Auxiliary Lemmata for Theorem \ref{thm:l1-all-combined}}
\label{subsec:proofs-auxiliary-l1-lemmas}

\subsubsection{Proof of Lemma \ref{lem:lp-for-l1}}

    We check each condition from \cref{ass:sufficient-conditions-psi}. We already proved \ref{ass:suff-1},\ref{ass:suff-2} and \ref{ass:suff-3} in the proof of \cref{prop:lp-norm-squared}.

    \paragraph{\ref{ass:suff-4}.} We use the norm inequalities $\norm{\cdot}_p \leq \norm{\cdot}_1 \leq d^{1-1/p}\norm{\cdot}_p$ and the fact that if $1<p\leq 1+1/\log(d)$ 
    \begin{equation*}
        d^{1-1/p}\leq d^{\frac{1}{1+\log d}}=\exp\pr{\frac{\log d}{1+\log d}} \leq \exp(1)=e.
    \end{equation*}
    which implies that for all $\alpha\in\RR^d$ and $\alpha'\in \tau B_1^d$
    \begin{align*}
        \varphi_K(\alpha)&=\norm{\alpha}_1\leq e \norm{\alpha}_p = c_l\sqrt{\psi(\alpha)} \\
        \sqrt{\psi(\alpha')} &=\norm{\alpha'}_p \leq \norm{\alpha'}_1 \leq c_u \tau
    \end{align*}
    with $c_l=e,c_u=1$.
This concludes the proof.

\subsubsection{Proof of Lemma \ref{lem:l1-moreau}}

    We first note that the potential in \cref{lem:l1-moreau} is indeed given by the Moreau envelope, since by, e.g., \cite[Example 4.2]{Beck2012smoothing} it holds
    \begin{equation*}
        (\cM_\lambda\norm{\cdot}_1)(\alpha)=\sum_{i=1}^d h_\lambda(\alpha_i) \quad \text{with} \quad h_\lambda (x)= \begin{cases} \frac{x^2}{2\lambda} & \abs{x}\leq \lambda, \\ \abs{x}-\frac{\lambda}{2} & \abs{x}> \lambda .\end{cases}
    \end{equation*}
    It follows that $\cM_\lambda \norm{\cdot}_1+\frac{d\lambda}{2}\geq \norm{\cdot}_1\geq 0$. We check that \cref{ass:sufficient-conditions-psi} is indeed satisfied.

    \paragraph{\ref{ass:suff-1}.} First note that the function is differentiable, as Moreau envelopes are continuously differentiable \citep[\S 3.1]{Parikh2014proximal} or \citep[\S 4.2]{Beck2012smoothing}. Moreover, a calculation shows that
    \begin{equation*}
        \nabla\psi(\alpha) = 2\pr{\sum_{i=1}^d h_\lambda(x)+\frac{d\lambda}{2}}\pr{h'_\lambda(\alpha_1),\dots,h'_\lambda(\alpha_d)} + \rho \alpha \quad \text{with} \quad h'_\lambda(x)=
        \begin{cases}
            \frac{x}{\lambda} & \abs{x}\leq \lambda \\
            \sign{x} & \abs{x}\geq \lambda
        \end{cases}
    \end{equation*}
    and so $\nabla\psi(0)=0$.

    \paragraph{\ref{ass:suff-2}.} The Moreau envelope $\cM_\lambda \norm{\cdot}_1+\frac{d\lambda}{2}$ is convex \citep[\S 4.2]{Beck2012smoothing} and non-negative, so that
    \begin{equation*}
        \sqrt{\psi(\alpha)} = \sqrt{\pr{(\cM_\lambda \norm{\cdot}_1)(\alpha)+\frac{d\lambda}{2}}^2+\frac{\rho}{2}\norm{\alpha}_2^2}
    \end{equation*}
    is convex, as it is the Euclidean norm of two non-negative convex functions.

    \paragraph{\ref{ass:suff-3}.} E.g., by \cite{Nikodem2011characterization}, the $\rho$-strong convexity of $\psi$ with respect to $\ell_2$-norm holds if and only if
    \begin{equation*}
        \psi-\frac{\rho}{2}\norm{\alpha}_2^2 = \pr{(\cM_\lambda \norm{\cdot}_1)(\alpha)+\frac{d\lambda}{2}}^2
    \end{equation*}
    is convex. Since $\cM_\lambda \norm{\cdot}_1$ is convex \citep[\S 4.2]{Beck2012smoothing}, $\cM_\lambda \norm{\cdot}_1+\frac{d\lambda}{2}\geq 0$ and squaring is non-decreasing, this is the case.

    \paragraph{\ref{ass:suff-4}.} Finally, we have that for all $\alpha\in\RR^d$, $\alpha'\in \tau B_1^d$ and $\rho=2$, $\lambda\leq 2 \tau /d$ that
    \begin{align*}
        \varphi_K(\alpha)&= \norm{\alpha}_1 \leq (\cM_\lambda \norm{\cdot}_1)(\alpha)+\frac{d\lambda}{2}\leq \sqrt{\pr{(\cM_\lambda \norm{\cdot}_1)(\alpha)+\frac{d\lambda}{2}}^2+\frac{\rho}{2}\norm{\alpha}_2^2} =c_l\sqrt{\psi(\alpha)}, \\
        \sqrt{\psi(\alpha')}&=\sqrt{\pr{(\cM_\lambda \norm{\cdot}_1)(\alpha')+\frac{d\lambda}{2}}^2+\frac{\rho}{2}\norm{\alpha'}_2^2} \leq \sqrt{\pr{\norm{\alpha'}_1+\tau}^2+\norm{\alpha'}_2^2} \leq \sqrt{5}\tau = c_u\tau,
    \end{align*}
    with $c_l=1, c_u=\sqrt{5}$. Here we used that for any Moreau envelope, $\cM_\lambda f \leq f$ by definition.

\subsubsection{Proof of Lemma \ref{lem:squared-hypentropy}}

Recall that we defined the function
\begin{align*}
\psi(\alpha)&=\pr{\arcsinh(\gamma^{-1})^{-1}\sum_{i=1}^d \alpha_i \arcsinh(\alpha_i/\gamma)-\sqrt{\alpha_i^2+\gamma^2}+\gamma+1}^2\\
    &=\pr{\sum_{i=1}^d g_\gamma(\alpha_i)}^2 \\
    &=f^2(\alpha)
\end{align*}
where we now introduced the new notation
\begin{equation*}
    f=\sqrt{\psi} \quad \text{and} \quad g_\gamma(x):= \arcsinh\pr{\gamma^{-1}}^{-1}\pr{ x\arcsinh\pr{x/\gamma}-\sqrt{x^2+\gamma^2}+\gamma+1}.
\end{equation*}

We prove that $\psi$ satisfies \cref{ass:sufficient-conditions-psi}.

\paragraph{\ref{ass:suff-1}.} $\psi$ is twice continuously differentiable with gradient
\begin{equation*}
    \nabla \psi(\alpha)= 2\arcsinh(\gamma^{-1})^{-2}f(\alpha)
    \begin{pmatrix}
        \arcsinh(\alpha_1/\gamma)\\
        \vdots \\
        \arcsinh(\alpha_d/\gamma)
    \end{pmatrix} \implies \nabla \psi(0)=0
\end{equation*}
This follows by straight-forward calculations \citep{Ghai2020}.

\paragraph{\ref{ass:suff-2}.}
$\sqrt{\psi}=f$ is (strictly) convex, as a simple calculation \citep{Ghai2020} shows that the Hessian of $f$ is given by
\begin{equation*}
    \nabla^2 f(x) = \arcsinh(\gamma^{-1})^{-1} \diag\pr{\frac{1}{\sqrt{x_1^2+\gamma^2}} \dots \frac{1}{\sqrt{x_d^2+\gamma^2}}}
\end{equation*}
which is positive definite everywhere.

To proceed, we state two useful facts.
\begin{itemize}
    \item 
\emph{Fact 1: For all $\gamma>0$ and $\alpha\in \RR^d$, $f(\alpha)\geq \norm{\alpha}_1\vee d \arcsinh (\gamma^{-1})^{-1}$.}

We first show that $f(\alpha)\geq \norm{\alpha}_1$. It suffices to show that $h_\gamma(x)=g_\gamma(x)-\abs{x}\geq 0$ for all $x\in\RR$. First of all, $h_\gamma(0)=\arcsinh\pr{\gamma^{-1}}^{-1}>0$ for all $\gamma>0$, and due to the symmetry of $h_\gamma$, it therefore suffices to show that $ h_\gamma(x)\geq 0$ for all $x\in(0,\infty)$. On this interval, $x\mapsto h_\gamma(x)$ is twice differentiable with
\begin{equation*}
    \frac{dh_\gamma}{dx}(x)=\arcsinh\pr{\gamma^{-1}}^{-1}\arcsinh\pr{x/\gamma}-1 \quad \text{and}\quad \frac{d^2h_\gamma}{(dx)^2}(x)= \frac{\arcsinh\pr{\gamma^{-1}}^{-1}}{\sqrt{x^2+\gamma^2}}.
\end{equation*}
Since $\frac{d^2h_\gamma}{(dx)^2}>0$ for all $x\in(0,\infty)$, we know that $h_\gamma$ is strictly convex on $(0,\infty)$. Therefore, $h_\gamma$ has a unique global minimum on $(0,\infty)$, as $\lim_{x\to\infty}h_\gamma(x)=\infty$ and $\frac{dh_\gamma}{dx}(x)<0$ for small enough $x>0$. The first-order condition of optimality $\frac{dh_\gamma}{dx}(x^\star)=0$ yields $x^\star=1$, and thus we know that for all $x\in(0,\infty)$
\begin{align*}
    h_\gamma(x)&\geq h_\gamma(x^\star)\\
    &=\arcsinh\pr{\gamma^{-1}}^{-1}\pr{ \arcsinh\pr{1/\gamma}-\sqrt{1^2+\gamma^2}+\gamma+1}-1\\
    &=\frac{\gamma+1-\sqrt{1+\gamma^2}}{\arcsinh(1/\gamma)}>0.
\end{align*}
hence, $f(\alpha)\geq \norm{\alpha}_1$.
Now, we can also see that for all $x\in\RR$
\begin{equation*}
    \frac{dg_\gamma}{dx}(x)=\arcsinh\pr{\gamma^{-1}}^{-1}\arcsinh\pr{x/\gamma} \quad \text{and} \quad \frac{d^2h_\gamma}{(dx)^2}(x)= \frac{\arcsinh\pr{\gamma^{-1}}^{-1}}{\sqrt{x^2+\gamma^2}}.
\end{equation*}
Hence, $g_\gamma$ has a unique global minimum at $x^\star = 0$, and thus for all $x\in\RR$, $g_\gamma(x)\geq g_\gamma(0) = \arcsinh (\gamma^{-1})^{-1}$, and thus
$f(\alpha)=\sum_{i=1}^d g_\gamma(\alpha_i)\geq d \arcsinh (\gamma^{-1})^{-1}$.

\item \emph{Fact 2: $\gamma \arcsinh(1/\gamma)\leq 1$ for all $\gamma>0$.}

    This follows by observing that $\frac{d^2}{(dx)^2} x\arcsinh(1/x)=-\sqrt{1/x^2+1}/(x^2+1)^2 <0$, which means its concave, and $\frac{d}{dx} x\arcsinh(1/x) =\arcsinh(1/x)-1/(x\sqrt{1/x^2+1})>0$, which means it has no maximizer, and
    \begin{equation*}
        \lim_{x\to 0} x\arcsinh(1/x) = 0, \quad \lim_{x\to \infty} x\arcsinh(1/x) = 1.
    \end{equation*}
\end{itemize}

\paragraph{\ref{ass:suff-3}.} We follow similar arguments used in \cite{Ghai2020} and will apply \cite[Theorem 3]{Yu2015strongconvexity} to show that $\psi$ is $\rho_{\gamma}$-strongly convex on $\RR^d$ with respect to $\norm{\cdot}_1$ and $\norm{\cdot}_2$, where $\rho_{\gamma} = \arcsinh(\gamma^{-1})^{-2}$.

First note that the Hessian of $\psi$ is given by
\begin{equation*}
    \nabla^2 \psi(x)=2\arcsinh(\gamma^{-1})^{-2}\br{
    \begin{pmatrix}
        \arcsinh(x_1/\gamma)\\
        \vdots \\
        \arcsinh(x_d/\gamma)
    \end{pmatrix}
    \begin{pmatrix}
        \arcsinh(x_1/\gamma)\\
        \vdots \\
        \arcsinh(x_d/\gamma)
    \end{pmatrix}^\top
    +f(x)
    \begin{pmatrix}
        \frac{1}{\sqrt{x_1^2+\gamma^2}} & & \\
        & \ddots & \\
        & & \frac{1}{\sqrt{x_d^2+\gamma^2}}
    \end{pmatrix}
    } 
\end{equation*} 
and hence we have that for any $y\in\RR^d$
\begin{equation*}
    y^\top \nabla^2\psi(x)y \geq 2\arcsinh(\gamma^{-1})^{-2} f(x)\sum_{i=1}^d \frac{y_i^2}{\sqrt{x_i^2+\gamma^2}}.
\end{equation*}
Using Fact 1 that $f(x)\geq \norm{x}_1 \vee d\arcsinh(\gamma^{-1})^{-1}$, we can bound
\begin{align*}
    \inf_{x,\norm{y}_1=1}f(x)\sum_{i=1}^d \frac{y_i^2}{\sqrt{x_i^2+\gamma^2}}&=\inf_{x,\norm{y}_1=1}\frac{f(x)}{\sum_{i=1}^d\sqrt{x_i^2+\gamma^2}}\pr{\sum_{i=1}^d \frac{y_i^2}{\sqrt{x_i^2+\gamma^2}}}\pr{\sum_{i=1}^d\sqrt{x_i^2+\gamma^2}} \\
    &\geq \inf_{x,\norm{y}_1=1} \frac{f(x)}{\sum_{i=1}^d\sqrt{x_i^2+\gamma^2}}\pr{\sum_{i=1}^d \sqrt{y_i^2}}\\
    &\geq \inf_{x} \frac{\norm{x}_1\vee d\arcsinh(\gamma^{-1})^{-1}}{\sum_{i=1}^d\sqrt{x_i^2+\gamma^2}} \\
    &\geq \inf_x \frac{\norm{x}_1\vee d\arcsinh(\gamma^{-1})^{-1}}{\norm{x}_1+d\gamma}\\
    &=\frac{ d\arcsinh(\gamma^{-1})^{-1}}{d\arcsinh(\gamma^{-1})^{-1}+d\gamma} =\frac{1}{1+\gamma \arcsinh(\gamma^{-1})} \geq  \frac{1}{2}
\end{align*}
where we used Fact 2 that $\gamma \arcsinh(1/\gamma)\leq 1$ for all $\gamma>0$. Therefore, $\inf_{x,\norm{y}_1=1} y^\top \nabla^2\psi(x)y \geq \arcsinh(\gamma^{-1})^{-2}$
which implies $\rho_\gamma$-strong convexity with respect to the $\ell_1$-norm \citep[Theorem 3]{Yu2015strongconvexity}.
With a similar argument, we have that
\begin{align*}
    \inf_{x,\norm{y}_2=1}f(x)\sum_{i=1}^d \frac{y_i^2}{\sqrt{x_i^2+\gamma^2}}&\geq \inf_{x}\frac{f(x)}{\max_i\sqrt{x_i^2+\gamma^2}} \\
    &\geq \inf_{x}\frac{\norm{x}_1\vee d\arcsinh(\gamma^{-1})^{-1}}{\norm{x}_1+\gamma} \\
    &=\frac{1}{1+\gamma \arcsinh(\gamma^{-1})/d} \geq \frac{1}{2}
\end{align*}
and therefore $\inf_{x,\norm{y}_2=1} y^\top \nabla^2\psi(x)y \geq \arcsinh(\gamma^{-1})^{-2}$.

\paragraph{\ref{ass:suff-4}.}
From Fact 1 it immediately follows that for any $\alpha\in\RR^d$ it holds
\begin{equation*}
    \varphi_K(\alpha)=\norm{\alpha}_1 \leq f(\alpha)=\sqrt{\psi(\alpha)}
\end{equation*}
so that $c_l=1$.

It remains to show the upper bound. For $\alpha'\in \tau B_1^d$ we have that $\abs{\alpha'_i}\leq \tau$ for all $i$. Hence,
\begin{align*}
    \sqrt{\psi(\alpha')}=f(\alpha')&=\sum_{i=1}^d \arcsinh\pr{\gamma^{-1}}^{-1}\pr{ \alpha'_i\arcsinh\pr{\alpha'_i/\gamma}-\sqrt{(\alpha'_i)^2+\gamma^2}+\gamma+1} \\
    &\leq \sum_{i=1}^d \arcsinh\pr{\gamma^{-1}}^{-1}\pr{ \alpha'_i\arcsinh\pr{\alpha'_i/\gamma}+1}\\
    &=\sum_{i=1}^d \arcsinh\pr{\gamma^{-1}}^{-1}\pr{ \abs{\alpha'_i}\log\pr{\frac{1}{\gamma}\pr{\sqrt{(\alpha'_i)^2+\gamma^2}+\abs{\alpha'_i}}}+1} \\
    &\leq \sum_{i=1}^d \arcsinh\pr{\gamma^{-1}}^{-1}\pr{ \abs{\alpha'_i}\log\pr{\frac{1}{\gamma}\pr{\sqrt{\tau^2+\gamma^2}+\tau}}+1} \\
    &= \norm{\alpha'}_1 \arcsinh\pr{\gamma^{-1}}^{-1}\log\pr{\frac{1}{\gamma}\pr{\sqrt{\tau^2+\gamma^2}+\tau}}+d\arcsinh\pr{\gamma^{-1}}^{-1} \\
    &\leq \tau \log\pr{\gamma^{-1}}^{-1}\log\pr{\frac{2\tau}{\gamma}+1}+d\arcsinh\pr{\gamma^{-1}}^{-1}
\end{align*}

For the first term, we notice
\begin{align*}
    \gamma\leq \frac{1}{\sqrt{2}}\wedge\frac{1}{4\tau}\implies \frac{2\tau}{\gamma}+1\leq \frac{1}{\gamma^2} \implies \log\pr{\frac{2\tau}{\gamma}+1} \leq 2\log(\gamma^{-1})  
\end{align*}
so we can bound the first term as $2\tau$. The second term can be bounded as
\begin{equation*}
    \gamma \leq \sinh(d/\tau)^{-1}\implies d\arcsinh(\gamma^{-1})^{-1}\leq d\arcsinh(\sinh(d/\tau))^{-1}=\tau
\end{equation*}
So overall, we get that $\sqrt{\psi(\alpha')}=f(\alpha')\leq 3\tau$, so we may set $c_u=3$.
This concludes the proof.

\subsubsection{Proof of Lemma \ref{lem:sigmoidal-approximation}}

We show that \cref{ass:sufficient-conditions-psi} holds by going through each point. We define the function
\begin{equation*}
    g_\gamma(x)=\frac{1}{\gamma}\pr{\log(1+\exp(-\gamma x))+\log(1+\exp(\gamma x))}
\end{equation*}
so that $\psi(\alpha) = \pr{\sum_{i=1}^d g_\gamma(\alpha_i)}^2$.

\paragraph{\ref{ass:suff-1}.} The fact that the potential is twice differentiable is clear, and the gradient is given by \citep{Schmidt2007L1approximation}
\begin{equation*}
    \nabla \psi(\alpha) =2\pr{\sum_{i=1}^d g_\gamma(\alpha_i)}\pr{g_\gamma'(\alpha_1),\dots,g_\gamma'(\alpha_d)}^\top \quad \text{with} \quad g_\gamma'(x)= \frac{1}{1+\exp\pr{-\gamma x}}-\frac{1}{1+\exp\pr{\gamma x}}
\end{equation*}
so that $\nabla \psi(0)=0$.

\paragraph{\ref{ass:suff-2}.} We have by \cite{Schmidt2007L1approximation} that
\begin{equation*}
    \nabla^2 \sqrt{\psi(\alpha)} = \diag\pr{\frac{d^2}{(dx)^2}g_\gamma(\alpha_1),\dots, \frac{d^2}{(dx)^2}g_\gamma(\alpha_d)} \quad \text{with} \quad \frac{d^2}{(dx)^2}g_\gamma(x) = \frac{2\gamma \exp(\gamma x)}{\pr{1+\exp(\gamma x)}^2},
\end{equation*}
and so all its eigenvalues (strictly) greater than zero everywhere, implying that $\sqrt{\psi}$ is convex.

\paragraph{\ref{ass:suff-3}.} Strict convexity of $\psi$ follows analogously, as
\begin{equation*}
    \nabla^2 \psi(\alpha) = 2\pr{\nabla\sqrt{\psi(\alpha)}\nabla\sqrt{\psi(\alpha)}^\top + \sqrt{\psi(\alpha)}\nabla^2\sqrt{\psi(\alpha)}}
\end{equation*}
also has all its eigenvalues strictly greater than zero, implying that $\psi$ is strictly convex.

\paragraph{\ref{ass:suff-4}.}
We can easily show that this approximation satisfies the conditions. First, notice that it holds $g_\gamma(0)=\log(4)/\gamma$, and for $x>0$
\begin{equation*}
    \frac{d}{dx}\pr{g_\gamma(x)-\abs{x}}=\frac{1}{1+\exp\pr{-\gamma x}}-\frac{1}{1+\exp\pr{\gamma x}}-1<0
\end{equation*}
and by symmetry of $g_\gamma$ for $x<0$ we have $\frac{d}{dx}\pr{g_\gamma(x)-\abs{x}}>0$. Combining this with $\lim_{x\to\infty}g_\gamma(x)-\abs{x} = 0$ and $\lim_{x\to-\infty}g_\gamma(x)-\abs{x} = 0$ yields $g_\gamma(x)\in[\abs{x},\abs{x}+\log(4)/\gamma]$ for all $x\in\RR$.
Therefore, for all $\alpha\in \RR^d$ and all $\alpha'\in \tau B_1^d$, we have for $\gamma \geq d\log (4)/\tau$
\begin{align*}
    \varphi_K(\alpha)&=\norm{\alpha}_1 \leq \sum_{i=1}^dg_\gamma(\alpha_i)=c_l\sqrt{\psi} \\
    \sqrt{\psi(\alpha')}&= \sum_{i=1}^dg_\gamma(\alpha_i') \leq \sum_{i=1}^d \pr{\abs{\alpha_i'} + \frac{\log 4}{\gamma}} \leq \norm{\alpha'}_1+\frac{d\log 4}{\gamma} \leq c_u\tau
\end{align*}
with $c_l=1$ and $c_u=2$. This concludes the proof.

\end{document}